\documentclass{ecai}

\usepackage{balance} 
\usepackage{amsmath}
\usepackage{amsthm}
\usepackage{amssymb}
\usepackage{multicol}
\usepackage{tcolorbox}
\usepackage{colortbl}
\usepackage{todonotes}
\usepackage{comment}
\usepackage{multicol}
\usepackage{tcolorbox}
\usepackage{multirow}

\newtheorem{prop}{Proposition}
\newtheorem{proposition}[prop]{Proposition}
\newtheorem{defn}{Definition}
\newtheorem{definition}[defn]{Definition}

\newtheorem{exmp}{Example}
\newtheorem{example}[exmp]{Example}

\newtheorem{theorem}{Theorem}

\newtheorem{property}{Property}

\newcommand{\bS}{{\mathbf{S}}}
\newcommand{\bQ}{{\mathbf{Q}}}
\newcommand{\bL}{{\mathbf{L}}}
\newcommand{\F}{{\mathtt{F}}}
\renewcommand{\d}{{\mathtt{dom}}}

\newcommand{\C}{\mathtt{C}}
\newcommand{\D}{\mathtt{D}}

\newcommand{\mbb}[1]{\ensuremath\mathbb{#1}}

\newcommand{\core}{\mathtt{Core}}
\newcommand{\lit}{\mathtt{Lit}(\mbb{T})}

\newcommand{\bLwf}{{{\mathbf{L_{wf}}}}}
\newcommand{\bLa}{{\mathbf{L_c}}}

\newcommand{\necp}{{\mathtt{gNec}}}

\newcommand{\wL}{{\mathtt{sNec}}}

\newcommand{\feat}{{\mathtt{Feat}}}

\newcommand{\sufc}{{\mathtt{gSuf}}}

\newcommand{\ssuf}{{\mathtt{sSuf}}}
\newcommand{\csuf}{{\mathtt{cSuf}}}

\newcommand{\dd}{{\mathtt{d}}}

\newenvironment{exs}[1]{\begin{trivlist}\item[]\textbf{Ex. \ref{#1} (Cont)}}{\end{trivlist}} 

\newcommand{\BibTeX}{\rm B\kern-.05em{\sc i\kern-.025em b}\kern-.08em\TeX}

\begin{document}
 \begin{frontmatter}
\title{Axiomatic Foundations of Counterfactual Explanations}

\author[A]{\fnms{Leila}~\snm{Amgoud}\thanks{Corresponding Author. Email: leila.amgoud@irit.fr}}
\author[B]{\fnms{Martin}~\snm{Cooper}}

\address[A]{IRIT, CNRS, France}
\address[B]{IRIT, Toulouse University, France}




\begin{abstract}
Explaining autonomous and intelligent systems is critical in order to  
improve trust in their decisions. 
\textit{Counterfactuals}  
have emerged as one of the most compelling forms of explanation. They address “why not” questions 
by revealing how decisions could be  altered. 
Despite the growing literature, most existing explainers focus on a single \textit{type} of 
counterfactual and are restricted to \textit{local} explanations, focusing on individual instances. 
There has been no systematic study of alternative counterfactual types, nor of \textit{global} 
counterfactuals that shed light on a system’s overall reasoning process. 

This paper addresses the two gaps by introducing an axiomatic framework built on a set of desirable properties for counterfactual explainers. It proves impossibility theorems showing that no single explainer can satisfy certain axiom combinations simultaneously, and fully characterizes all compatible sets. Representation theorems then establish five one-to-one correspondences between specific subsets of axioms and the families of explainers that satisfy them. Each family gives rise to a distinct type of counterfactual explanation, uncovering five fundamentally different types of counterfactuals. Some of these correspond to local explanations, while others capture global explanations. Finally, the framework situates existing explainers within this taxonomy, formally characterizes their behavior, and analyzes the computational complexity of generating such explanations.
\end{abstract}

\end{frontmatter}








\section{Introduction}

Explaining the decisions of autonomous and intelligent systems has attracted significant attention over the past decade (see \cite{Miller2019, molnar20} for surveys). The overarching goal of this research is to enhance trust in the decisions made by such systems. 
Two primary types of explanations have been extensively studied. The first type addresses “why” questions by identifying the key factors that led to a specific decision (e.g., \cite{ecai23, darwiche20, ICML21}). For example:
“Alice is not eligible for a loan because \textit{her annual income is} $\$$20K.” 
The second type—\textit{counterfactuals}—answers “why not” questions by describing how the decision could be changed. For instance: “If Alice’s annual income had been $\$$35K, she would have been eligible for a loan.”

Byrne argues in \cite{Byrne16, Byrne19} that individuals affected by decisions often prefer explanations that show how they can achieve better outcomes. This insight has motivated the development of numerous counterfactual explainers (see \cite{Karimi22, Guidotti24, JiangL0T24} for recent comprehensive surveys), applicable to both autonomous agents and multi-agent systems (e.g., \cite{RodriguezTW25, NooraniDHD25,ChenW0XHWX25,LeofanteL23}).

The vast majority of existing work has focused on \textit{local counterfactuals}—explanations tailored to individual input instances—and has typically considered only a \textbf{single type of counterfactual}. A counterfactual for a given input is defined as the "closest" alternative input that yields a different outcome. The various explainers differ in how they formalize this notion of closeness. To guide its definition and the selection of "reasonable" counterfactuals, several properties have been proposed in the literature (e.g., \cite{BodriaGGNPR23, Guido19, VermaBHHDS24, WinikoffTR25, WarrenSK22}). 
For instance, \textit{minimality} requires that a counterfactual involves the smallest possible change to the original input, while \textit{diversity} promotes generating distinct counterfactuals for the same input (see Section~\ref{related-work} for additional properties). Existing explainers are typically evaluated based on the extent to which they satisfy these properties and on the strategies they employ to generate counterfactuals.

Despite a rich body of work on counterfactual explanations, there has been no systematic study of alternative 
types of local counterfactuals. Moreover, global counterfactuals—which aim to capture and explain a model’s 
overall decision-making behavior—have received far less attention; to date, the only notable contribution is 
that of \cite{Amgoud22}. As a result, the theoretical trade-offs between local and global counterfactuals remain poorly understood.

This paper addresses these gaps by proposing a formal framework for systematically defining and comparing 
different types of local and global counterfactuals. Our approach is axiomatic—an established method in modern 
theoretical science that builds systems upon clearly stated axioms (basic properties) from which representation theorems can be derived.
Applying this approach to counterfactuals offers several benefits.
First, it brings \textbf{clarity} and \textbf{rigor} by precisely specifying desirable properties, fostering a shared 
understanding of what constitutes a sound counterfactual.
Second, it enables a \textbf{systematic comparison} of counterfactual types based on the axioms they satisfy, making their differences, strengths, and limitations explicit. This supports a principled and objective 
evaluation that goes beyond empirical performance alone.
Finally, the framework is \textbf{model-agnostic} and \textbf{general}, ensuring that the results are broadly 
applicable across diverse settings.

\vspace{0.1cm}

The contributions of this paper are sixfold and complement the existing literature on counterfactual explanations:

\begin{enumerate}
\item It introduces an \textit{axiomatic framework} grounded on a set of nine axioms.

\item It provides \textit{impossibility theorems}, showing that no explainer can satisfy certain combinations of axioms simultaneously.

\item It fully characterizes all compatible sets of axioms.

\item It presents \textit{representation theorems} that define one-to-one correspondences between specific subsets of axioms and the families of counterfactual explainers that satisfy them. These characterizations are valuable because they precisely and \textbf{exhaustively} describe \textbf{all} explainers that meet a given set of axioms. Key findings include:
\begin{itemize}
\item The nine axioms distinguish five distinct types of counterfactuals, each generated by a different family of explainers.
\item There are two fundamental forms of counterfactual explanation: \emph{necessary} and \emph{sufficient} reasons.
\item For each form (necessary, sufficient), one type corresponds to \emph{global} explanations, while the remaining types focus on \emph{local} explanations.
\item Among local sufficient reasons, there are two subtypes: \textit{sceptical} and \textit{credulous}.
\end{itemize}

\item It formally characterizes existing counterfactual explainers, showing that nearly all fall within the family generating local credulous sufficient reasons. The global explainers proposed in \cite{Amgoud22} correspond to the family of global necessary reasons. Thus, the framework uncovers three previously unrecognized types of counterfactuals.

\item It analyzes the computational complexity of generating explanations.
\end{enumerate}

\vspace{0.1cm}

\noindent The paper is organized as follows. 
Section~\ref{back} presents the necessary background, 
Section~\ref{axioms} defines axioms, 
Sections~\ref{necessary-reasons} and \ref{sufficient-reasons} are devoted to representation theorems. 
Section~\ref{comparison} compares the types, 
Section~\ref{related-work} discusses related work,  and 
Section~\ref{complexity} provides complexity results. 
We give proofs of all propositions and theorems from the paper
in the appendix. 


\section{Background}\label{back}

Throughout the paper, we consider a \textit{classification theory} as a tuple 
$\mbb{T} = \langle \F, \d, \C\rangle$ comprising a finite set $\F$ of \textit{features}, a function 
$\d$ which returns the \textit{domain} of every feature $f \in \F$ such that $\d(f)$ is finite and $|\d(f)| > 1$, 
and a finite set $\C$ of \textit{classes} with $|\C| \geq 2$.  
Let $\lit$ be the set of all literals 
in $\mbb{T}$, where a \emph{literal} is a pair $(f,v)$ where $f \in \F$ and $v \in \d(f)$. 
%
%
A \textit{partial assignment} is any set of literals with each feature in 
$\F$ occurring at most once; 
it is called an \textit{instance} when every feature appears once. 
We denote by $\mbb{E}(\mbb{T})$ the set of all possible partial assignments and by 
$\mbb{F}(\mbb{T})$ the \textit{feature space}, i.e., the set of all instances, of theory $\mbb{T}$.  
For a set $E \in \mbb{E}(\mbb{T})$, $\mbb{E}(\mbb{T})\oslash E$ denotes the set of all 
partial assignments that do not contain any literal from $E$, i.e., 
for any $X \in \mbb{E}(\mbb{T})\oslash E$, $X \in \mbb{E}(\mbb{T})$ and $E \cap X = \emptyset$.
For $E \in \mbb{E}(\mbb{T})$ and $x \in \mbb{F}(\mbb{T})$, we denote by $x_{\downarrow E}$ the partial assignment such that $x_{\downarrow E} = E \cup \{(f,v) \in x \mid \nexists  v' \in \d(f) \mbox{ s.t. } v' \neq v \mbox{ and } (f, v') \in E\}$. The idea is to replace the values of common features by those in $E$. Consider a theory consisting of two binary features, the instance $x = ((f_1,0), (f_2,0))$ and 
$E = \{(f_2,1)\}$. Then, $x_{\downarrow E} =  ((f_1,0), (f_2,1))$.

\begin{property}\label{propy:instances}
Let $\mbb{T}$ be classification theory. $\forall x \in \mbb{F}(\mbb{T})$, 
$\forall E \in \mbb{E}(\mbb{T})$,  
$x_{\downarrow E} \in  \mbb{F}(\mbb{T})$.
\end{property}

\noindent For any $x \in \mbb{F}(\mbb{T})$ and any $E \in \mbb{E}(\mbb{T})$, $x\circleddash E = \{y \in \mbb{F}(\mbb{T}) \mid x\setminus y = E\}$.
Consider a theory consisting of two  
features, where $\d(f_1) = \{0,1\}$ and $\d(f_2) = \{0,1,2\}$, the instance $x = ((f_1,0), (f_2,0))$ 
and $E = \{(f_2,0)\}$. Then, $x\circleddash E =  \{y, z\}$ where 
$y = ((f_1,0), (f_2,1))$ and $z = ((f_1,0), (f_2,2))$. 

A \textit{classifier} on a theory $\mbb{T}$ is a \textbf{surjective} function $\kappa : \mbb{F}(\mbb{T}) \rightarrow \C$ mapping every instance to a class such that every class is assigned to at least one instance. 

Among literals, we distinguish those that are \textit{core}, or mandatory, to a given class under a given classifier. 

\begin{definition}\label{def:core}
Let $\mbb{T} = \langle \F, \d, \C\rangle$ be a theory, $c \in \C$, $\kappa$ a classifier on $\mbb{T}$, $l \in \lit$.  
$l$ is \emph{core} to $c$ iff $\forall x \in \mbb{F}(\mbb{T})$ s.t. $\kappa(x) = c$, $l \in x$. 
$\core(c)$ is the set of core literals of $c$.
\end{definition}

The set $\core(c)$, for $c \in \C$, may be empty as shown below. 

\begin{example}\label{ex1bis}
	Let $\kappa$ be a classifier that predicts where to take a friend. It is a function of the temperature $t$ and 
	the friend's favourite activity $a$, with 
	$\d(t) = \{\text{hot, mild, freezing}\}$, 
	$\d(a) = \{\text{climbing, reading}$, $\text{skiing}\}$, and 
	$\C = \{\text{beach, mountain, cinema}\}$.
	$\kappa$ predicts the beach on hot days, the mountain if it is mild and the friend likes climbing
	or it is freezing cold and she likes skiing, otherwise the cinema. 
    Below are all instances and their decisions. 
     \begin{center}
		\begin{tabular}{c|cc|c}\hline
			$\mbb{F}(\mbb{T})$  & $t$      &        $a$            & $\kappa(x_i)$   \\\hline
\rowcolor{orange!10} 			$x_1$    & $\text{hot}$        & $\text{climbing}$     & 
    $\text{beach}$ \\
\rowcolor{orange!10} 			$x_2$    & $\text{mild}$       & $\text{climbing}$     & $\text{mountain}$ \\          
\rowcolor{orange!10} 			$x_3$    & $\text{freezing}$   & $\text{reading}$      & $\text{cinema}$ \\
			$x_4$    & $\text{freezing}$   & $\text{skiing}$       & $\text{mountain}$ \\
			$x_5$    & $\text{freezing}$   & $\text{climbing}$     & $\text{cinema}$ \\
		    $x_6$    & $\text{hot}$        & $\text{reading}$      & $\text{beach}$ \\
		    $x_7$    & $\text{hot}$        & $\text{skiing}$       & $\text{beach}$ \\ 
		    $x_8$    & $\text{mild}$       & $\text{reading}$      & $\text{cinema}$ \\
		    $x_9$    & $\text{mild}$       & $\text{skiing}$       & $\text{cinema}$ \\	
			\hline
		\end{tabular}
	\end{center}
  
\noindent The following lists the set of core literals corresponding to each class.
\begin{itemize}
    \item $\core(\text{beach}) = \{(t,\text{hot})\}$, 
    \item $\core(\text{mountain}) = \emptyset$,
    \item $\core(\text{cinema}) = \emptyset$.
\end{itemize}
\end{example}

Let $X$ be a set of objects. 
A \emph{preordering} on $X$ is a binary relation $\succeq$ on $X$ that is \textit{reflexive} and \textit{transitive}.  It is \emph{total} iff for all $x, y \in X$, $x \succeq y$ or $y \succeq x$. 
The notation $x \succeq y$ stands for $x$ is at least as preferred as $y$; $x \succ y$ is a shortcut for $x \succeq y$ and $y \not\succeq x$. 
Let $\max(X, \succeq) = \{E \in X \mid \nexists E' \in X \mbox{ such that } E' \succ E\}$ be the set of most preferred 
elements of $X$.
Finally, a \emph{weighting} on the set $X$ is a function $\sigma: X \rightarrow [0, +\infty)$.
\begin{table*}
\begin{center}	
\begin{tabular}{lll}
\hline\hline 
& \textit{Success}    & $\forall \bQ = \langle \mbb{T}, \kappa, x\rangle$, 
$\bL(\bQ) \neq \emptyset$.  \\

$\makebox(0,0){\rotatebox[origin=c]{90}{ \ \ General }} \ \ $

& \textit{Non-Triviality} & $\forall \bQ = \langle \mbb{T}, \kappa, x\rangle$, $\emptyset \notin \bL(\bQ)$.  \\ 

&  \textit{Equivalence}    & $\forall \bQ = \langle \mbb{T}, \kappa, x\rangle$, 
                             $\forall \bQ' = \langle \mbb{T}, \kappa, x'\rangle$, 
                        if $\kappa(x) = \kappa(x')$, then $\bL(\bQ) = \bL(\bQ')$. \\
\hline\hline
			
& \textit{Feasibility}  & $\forall \bQ = \langle \mbb{T}, \kappa, x\rangle$, $\forall E \in \bL(\bQ)$,   $E \subseteq x$.  \\

$\makebox(0,0){\rotatebox[origin=c]{90}{\ \ \ \shortstack{Nec. \vspace{-0.7mm} \\ Reasons} }}$

& \textit{Coreness} & $\forall \bQ = \langle \mbb{T}, \kappa, x\rangle$, 
                      $\forall E \in \bL(\bQ)$, $E \subseteq \core(\kappa(x))$. \\

& \textit{Sceptical Validity}     & $\forall \bQ = \langle \mbb{T}, \kappa, x\rangle$, 
			                $\forall E \in \bL(\bQ)$, $\forall y \in x\circleddash E$, 
                            $\kappa(y) \neq \kappa(x)$. \\  
                           
\hline\hline

& \textit{Novelty}     &    $\forall \bQ = \langle \mbb{T}, \kappa, x\rangle$, 
			                $\forall E \in \bL(\bQ)$, $x \cap E = \emptyset$. \\

$\makebox(0,0){\rotatebox[origin=c]{90}{\ \ \ \shortstack{Suff. \\ Reasons} \ }}$

& \textit{Strong Validity}     & $\forall \bQ = \langle \mbb{T}, \kappa, x\rangle$, 
			                $\forall E \in \bL(\bQ)$, $\nexists y \in \mbb{F}(\mbb{T})$ s.t. 
                            $E \subseteq y$ and $\kappa(y) = \kappa(x)$. \\   

& \textit{Weak Validity}  & $\forall \bQ = \langle \mbb{T}, \kappa, x\rangle$, 
$\forall E  \in \bL(\bQ)$, $\kappa(x_{\downarrow E}) \neq \kappa(x)$. \\

\hline\hline
\end{tabular}
\caption{Axioms for explainer $\bL$.}
\label{tab:axioms}
\end{center}	
\end{table*}

\section{Axiomatic Framework}\label{axioms}

We introduce an axiomatic framework that defines counterfactual explainers$-$functions generating counterfactual explanations. Before presenting the axioms, we begin by introducing the notion of a \emph{query}, which represents the question posed to an explainer.   

\begin{definition}\label{Queries}
A \emph{query} is a tuple $\bQ = \langle \mbb{T}, \kappa, x\rangle$ such that $\mbb{T}$ is a classification theory, 
$\kappa$ a classifier on $\mbb{T}$ and $x \in \mbb{F}(\mbb{T})$. 
\end{definition}

A \textit{counterfactual explainer}—or explainer for short—is a function that maps each query to a set of partial assignments, each representing a \textit{counterfactual}. As we  will see in the following sections, there is no single definition of a counterfactual. Nevertheless, it is widely acknowledged in the literature that all such definitions aim to provide insights into how a prediction made by a classifier (or an AI system) can be altered.  

\begin{definition}\label{explainer}
A \emph{counterfactual explainer} is a function $\bL$ mapping every query  $\bQ = \langle \mbb{T}, \kappa, x\rangle$ into  a subset of $\mbb{E}(\mbb{T})$. 
Every $E \in \bL(\bQ)$ is a \emph{counterfactual}.
\end{definition}

\noindent \textbf{Notation:} Let $\bL$, $\bL'$ be two explainers. We write 
$\bL \sqsubseteq \bL'$ to denote that for any query $\bQ$, 
$\bL(\bQ) \subseteq \bL'(\bQ)$. 

\vspace{0.15cm}
Table~\ref{tab:axioms} introduces axioms—formal properties of explainers—three of which, Success, Non-Triviality, and Coreness, are borrowed from \cite{Amgoud22}, while the remaining ones are novel. The axioms are grouped into three categories: the first contains general properties that any explainer might satisfy, while the other two define distinct forms of counterfactuals as will be shown in next sections.

\vspace{0.1cm}
\emph{Success} guarantees at least one explanation per query;  
\emph{Non-Triviality} rejects empty explanations since uninformative; and 
\emph{Equivalence} ensures that instances with the same decisions get the same explanations. This axiom is suitable for explainers that produce global explanations—i.e., those that concern classes rather than input instances. 
\emph{Feasibility} states that an explanation must be part of the instance whose decision is being explained. 
As we will see later, for some type of counterfactuals, this axiom helps identify the features whose values need to be changed in order to alter the instance’s prediction.
\emph{Coreness} goes further by ensuring that an explanation contains only core literals for the decision. The logic is that missing a core literal certainly leads to another class. 
\emph{Sceptical validity} ensures that an explanation $E$ surely leads to a change in a decision. Indeed, every instance that differs from the one being explained only in $E$ 
has a different class. 
%
%
\emph{Novelty} ensures that an explanation does not share a literal with the original instance.  The idea is that in order to inform a user how to change an output, the  explanation must involve modifications to the input. This axiom 
concerns some types of counterfactuals. 
\emph{Strong validity} is similar in spirit to Sceptical validity as it ensures that an explanation $E$ must lead to a change in a decision. Indeed, any instance that contains $E$ has a different prediction.  
\emph{Weak Validity} is less demanding, since it only requires that the instance obtained by replacing the new literals in the original instance is of a different class. 
.

\vspace{0.1cm}
The axioms are not all independent, as demonstrated below. Nevertheless, each plays a crucial role in defining distinct types of counterfactual explanations.

\begin{proposition}\label{links}
The following implications hold. 
\begin{itemize}
\item Coreness $\Rightarrow$ Feasibility. 
\item Coreness and Non-Triviality $\Rightarrow$ Sceptical Validity.
\item Sceptical Validity $\Rightarrow$ Non-Triviality.
\item Weak Validity $\Rightarrow$ Non-Triviality.  
\item Strong Validity $\Rightarrow$ Weak Validity. 
\item Novelty and Non-Triviality $\Rightarrow$ Sceptical Validity.
\end{itemize}
\end{proposition}

We present an impossibility theorem showing the \textit{incompatibility} of certain axioms—no explainer can satisfy all of them simultaneously. As we will show later, each compatible subset gives rise to a distinct form of counterfactual. 

\begin{theorem}\label{incompatibility}
The axioms of every set $(\mathtt{I}_i)$ are incompatible.
\begin{itemize} 
    \item $(\mathtt{I}_1)$ Success, Non-Triviality and Coreness. 
    \item $(\mathtt{I}_2)$ Success, Feasibility and Sceptical Validity.
    \item $(\mathtt{I}_3)$ Success, Novelty and Strong Validity.  
    \item $(\mathtt{I}_4)$ Success, Non-Triviality, Feasibility and Novelty. 
    \item $(\mathtt{I}_5)$ Success, Feasibility and Weak Validity.  
    \item $(\mathtt{I}_6)$ Success, Non-Triviality, \!Equivalence and Feasibility.
    \item $(\mathtt{I}_7)$ Success, Non-Triviality, Equivalence and Novelty.
\end{itemize}
\end{theorem}

All combinations of axioms not excluded by the two preceding results are compatible; in other words, each such combination is satisfied by at least one explainer.

\begin{theorem} \label{thm:allinc}
All combinations of axioms not disallowed by Proposition~1 and Theorem~1 are satisfied by some explainer.
\end{theorem}

Two fundamental questions emerge concerning these axioms: 
\begin{enumerate}
\item Can we precisely and exhaustively characterize the family of \textbf{all} explainers that satisfy a given 
      subset of the axioms?
\item Do the axioms lead to the same type of counterfactuals?
\end{enumerate}

In the next two sections, we present representation theorems that answer the first question affirmatively and reveal the existence of \textbf{five distinct types of counterfactuals}. These theorems show that counterfactual explanations fall into two fundamental forms—\textit{necessary} and \textit{sufficient}—each encompassing several distinct types. For clarity, we dedicate a separate section to each form.


\section{Necessary Reasons}\label{necessary-reasons}

There first form of counterfactuals can be expressed as follows: 

\begin{quote}
\hspace{-0.5cm} \textit{If it were not the case that $E$, the decision would have been different}. 
\end{quote}

For example, in instance $x_1$ of Example~\ref{ex1bis} 
if the temperature had not been hot, the decision would have been different,
as can be verified from the table.
Such explanations indicate the \textbf{part of an instance} that should be changed to 
avoid its decision. We call them  \emph{necessary reasons} since they are mandatory to 
guarantee the current decision. In this section, we study their \textit{global} and \textit{local} versions. 
The former explain a class while latter target an instance.

\subsection{Global Necessary Reasons}

Global necessary reasons explain the behavior of a classifier, that is, how it assigns 
classes. For a given class, they identify feature–value combinations that \textbf{distinguish it from 
all other classes}. In their absence, the classifier assigns a different class. 
 
\begin{definition}\label{nec-pro} 
Let $\bQ = \langle \mbb{T}, \kappa,x\rangle$ be a query.
A \emph{global necessary reason} (GNR) for $\kappa(x)$ is a set $E \in \mbb{E}(\mbb{T})$ such that: 
$$E \neq \emptyset \ \ \mbox{ and } \ \ \forall y \in \mbb{F}(\mbb{T}), \mbox{ if }  E \not\subseteq y, 
\mbox{ then } \kappa(y) \neq \kappa(x).$$
$\necp$ denotes the explainer that returns all GNRs.	
\end{definition} 

\begin{exs}{ex1bis}
Let $\bQ_i$ be the query on instance $x_i$, $i=1,2,3$. 
\begin{itemize}
\item $\necp(\bQ_1) = \{\{(t,\text{hot})\}\}$, 
\item $\necp(\bQ_2) = \necp(\bQ_3) = \emptyset$. 
\end{itemize}   
A hot temperature is characteristic of the \textit{beach} option, as it sets it apart from the mountain and cinema options. 
Thus, to change the decision for $x_1$, the value of $t$ must be altered to either mild or freezing. In contrast, the other two queries have no GNR, since every feature–value combination appears in at least two instances that are assigned different classes. For example, modifying the value of $a$ for $x_3$ does not necessarily lead to a class change, as 
$\kappa(x_5) = \kappa(x_3)$.
\end{exs}

We provide a representation theorem that characterizes the \textbf{entire family} of explainers generating GNRs. These are the \textbf{only ones} that satisfy both Coreness and Non-triviality.


\begin{theorem} 
An explainer $\bL$ satisfies Coreness and Non-triviality \  \textbf{iff} $\ \bL \sqsubseteq \necp$.	
\end{theorem}

Table~\ref{tab:axioms2} summarizes all the properties that are satisfied/violated by $\necp$ and any explainer 
$\bL$ such that $\bL \sqsubseteq \necp$. 

\begin{theorem} 
	The properties of Table~\ref{tab:axioms2} hold. 
\end{theorem}

These explainers violate Success since 
Theorem~\ref{incompatibility} shows that Coreness, Success and Non-Triviality are 
incompatible.  Since $\necp$ generates global explanations, it satisfies 
Equivalence: its explanations are common to all instances labelled by the same class. 
For example, $\necp(\bQ_1) = \necp(\bQ_6) = \necp(\bQ_7) = \{\{(t,\text{hot})\}\}$, 
where $\bQ_i$ concerns the instance $x_i$.

\vspace{0.1cm}

\noindent \textbf{To sum up,} the characterization result shows that the family of explainers satisfying the Coreness and Non-triviality axioms generates global necessary reasons for a class. This type of counterfactual highlights the factors that distinguish a given class from \textbf{all others}.

\subsection{Local Necessary Reasons}

We now turn to the problem of explaining the decision made for a specific instance.
A natural approach is to consider the global reasons associated with the 
instance’s class. However, this approach is overly rigid, as it requires exploring the entire feature space, which drastically reduces the chances of finding a  counterfactual (e.g., $x_2, x_3$ in Example~\ref{ex1bis}). 
In what follows, we introduce \textbf{sceptical necessary reasons} (SNR), an alternative approach that is less demanding. 
SNRs are subsets of an instance of interest $x$ that distinguish it from \textbf{any} 
other instance $y$ for which the decision would be different. Thus, their absence guarantees 
a different decision for $x$. 
Unlike the global approach, the local approach restricts the 
search space to instances that differ from $x$ only in candidate reasons. 

\begin{definition}\label{local-nec-pro}	
Let $\bQ = \langle \mbb{T}, \kappa,x\rangle$ be a query.
A \emph{sceptical necessary reason} (SNR) for $\kappa(x)$ is a set  $E \in \mbb{E}(\mbb{T})$ such that: 
\begin{itemize}
\item[] $E \subseteq x$  \qquad and \qquad  $\forall y \in x\circleddash E$, $\kappa(y) \neq \kappa(x).$
\end{itemize}
$\wL$ denotes the explainer that generates all SNRs. 
\end{definition}

\begin{exs}{ex1bis}
Let $\bQ_i = \langle \mbb{T}, \kappa,x_i\rangle$. 
\begin{itemize}
	\item $\wL(\bQ_1) = \{\{(t,\text{hot})\}, x_1\}$.
	\item $\wL(\bQ_2) = \{\{(t,\text{mild})\}, \{(a,\text{climbing})\}\}$.
	\item $\wL(\bQ_3) = \emptyset$.
\end{itemize} 
Note that  $x_2 \circleddash \{(t,\text{mild})\} =$ $\{x_1, x_5\}$,  
$\kappa(x_2) \neq \kappa(x_1)$, $\kappa(x_2) \neq \kappa(y)$. 
Hence, the SNR $\{(t,\text{mild})\}$ distinguishes 
$x_2$ from all instances that differ from it \textbf{only} in temperature.
\end{exs}

Note that global necessary reasons are sceptical ones. 
Consequently, if there are core literals, their non-empty subsets are 
SNRs. But, an explainer may return additional  
reasons whose literals are not core to the instance class (e.g., the case of $\bQ_2$).

\begin{proposition}\label{inec=>snr}
It holds that $\necp \sqsubseteq \wL$.    
\end{proposition}

We present a representation theorem showing an equivalence between the two 
axioms (Feasibility, Sceptical Validity) and the family of explainers that 
generate SNRs. 

\begin{theorem} 
An explainer $\bL$ satisfies Feasibility and Sceptical Validity \textbf{iff} $\bL \sqsubseteq \wL$.
\end{theorem}

\vspace{0.1cm}

Table~\ref{tab:axioms2} summarizes the axioms that are satisfied/violated by $\wL$ and, more generally, 
by any explainer $\bL \sqsubseteq \wL$. Note that they satisfy non-Triviality, hence they are \textbf{informative}. 
However, they violate Success and so might not guarantee explanations for queries (eg. $\bQ_3$). They also violate 
Coreness but satisfy its weaker version, Sceptical Validity. The latter guarantees a change in decision no matter 
what the new values of the features are, making SNRs \textbf{valid} in the sense of the property defined in \cite{Guidotti24}.


\vspace{0.1cm}

\noindent \textbf{To sum up,} Theorem~\ref{local-nec-pro} shows that the family of explainers satisfying Feasibility and Sceptical Validity generates local necessary reasons for an instance. Such reasons identify the factors that distinguish the instance from all others receiving a different decision.

\begin{table*}
\begin{center}	
{\small 
\begin{tabular}{|l|c|c|c|c|c|c|c|c|c|c|c|c}
\hline\hline 

 & \multicolumn{4}{|c}{Necessary Reasons}  &\multicolumn{6}{|c}{Sufficient Reasons} \\
  \hline 

Axioms                  
                  & $\!\! \bL \! \sqsubseteq \! \necp \!\!$ &   
                    $\! \necp \!$                           &   
                    $\!\! \bL \! \sqsubseteq \! \wL \!\!$   & 
                    $\! \wL \!$                             & 
                    $\!\! \bL \! \sqsubseteq \! \sufc \!\!$ & 
                    $\! \sufc \!$                           & 
                    $\!\! \bL \! \sqsubseteq \! \ssuf \!\!$ & 
                    $\! \ssuf \!$                           & 
                    $\!\! \bL \! \sqsubseteq \! \csuf \!\!$ & 
                    $\! \csuf \!$ \\\hline\hline 

\textit{Success}  &    $\times$ &  $\times$   & $\times$  & $\times$ & $-$ & $\checkmark$ & $\times$ & $\times$ & $-$ & $\checkmark$ \\

\textit{Non-Triviality} & $\checkmark$  &  $\checkmark$  & $\checkmark$ & $\checkmark$ & $\checkmark$ & $\checkmark$ & $\checkmark$ & $\checkmark$ & $\checkmark$ & $\checkmark$ \\ 

\textit{Equivalence}    & $-$  &  $\checkmark$ & $-$ & $\times$ & $-$ & $\checkmark$ & $-$ & $\times$ & $-$ & $\times$ \\ 

\hline\hline
			
\textit{Feasibility}    &  $\checkmark$  &   $\checkmark$  & $\checkmark$ & $\checkmark$ &  
$-$ & $\times$ & $-$ & $\times$ & $-$ & $\times$ \\ 

\textit{Coreness}           &   $\checkmark$ &   $\checkmark$  & $-$ & $\times$ & $-$  & $\times$ & $\times$ & $\times$ & $\times$ & $\times$ \\ 

\textit{Scep. Validity}     &  $\checkmark$  &   $\checkmark$  &  $\checkmark$ & $\checkmark$ & 
  $\checkmark$  & $\checkmark$ & $\checkmark$ & $\checkmark$ & $\checkmark$ & $\checkmark$ \\ 

\hline\hline
	
\textit{Novelty}     & $-$ &  $\times$   & $-$ & $\times$ &  $-$ & $\times$ & $\checkmark$ & $\checkmark$ & $\checkmark$ & $\checkmark$ \\

\textit{Strong Validity}     & $-$ &  $\times$    & $-$ & $\times$  & $\checkmark$ & $\checkmark$ & $\checkmark$ & $\checkmark$ & $-$ & $\times$ \\   

\textit{Weak Validity}     & $-$ &  $\times$   & $-$ & $\times$  & $\checkmark$ & $\checkmark$ & $\checkmark$ & $\checkmark$ & $\checkmark$ & $\checkmark$ \\ 
                            
\hline\hline
\end{tabular}
}
\caption{Formal Comparison of Explainers.  ``$-$'' means the axiom is satisfied by only some instances of the family but not all.}
\label{tab:axioms2}
\end{center}	
\end{table*}
\section{Sufficient Reasons}\label{sufficient-reasons}

The second form of counterfactuals can be expressed as follows: 
\begin{quote}
\hspace{-0.5cm}\textit{If it were the case that $E$, the decision would have been different.} 
\end{quote}

\noindent Unlike necessary reasons, which identify present characteristics whose absence 
would change the outcome, this second type focuses on absent characteristics whose presence 
would alter the outcome. For instance, "If Alice's annual income had been $\$35K$ (instead of the actual amount), she wouldn't have been denied a loan".  
We introduce the concept of \textit{sufficient reasons}—feature values 
that directly lead to a different decision—rather than identifying which parts of an instance must be modified.

\subsection{Global Sufficient Reasons}\label{global-suff}

We begin by analyzing the global behavior of a classifier and, consequently, of its classes. As previously noted, global necessary reasons capture what distinguishes one class from the others; however, such reasons may not always exist, as illustrated by instances $x_2$ and $x_3$ in Example~\ref{ex1bis}. To address this, we introduce the notion of \textit{global sufficient reasons} for \textit{avoiding} a class—factors whose presence leads the classifier to assign a different class.

\begin{definition}
\label{suff-con}	
Let $\bQ = \langle \mbb{T}, \kappa,x\rangle$ be a query. 
A \emph{global sufficient reason} (GSR) against $\kappa(x)$ is a set $E \in \mbb{E}(\mbb{T})$ such that:
$$\forall y \in \mbb{F}(\mbb{T}) \text{ with } E \subseteq y, \ \kappa(y) \neq \kappa(x).$$ 
Let $\sufc$ be the explainer that generates all GSRs.	
\end{definition}

\begin{exs}{ex1bis}
Let $\bQ_i = \langle \mbb{T}, \kappa,x_i\rangle$, $i=1,2,3$. 
\begin{itemize} 
\item $\sufc(\bQ_1) = \{\{(t,\text{mild})\}, \{(t,\text{freezing})\} \} \ \cup $ \\  
      $\{x_2, x_3, x_4,x_5,x_8,x_9\}$. 
       
\item $\sufc(\bQ_2) =  \{\{(t,\text{hot})\}, \{(a,\text{reading})\} \} \ \cup$ \\ $\{ x_1, x_3, x_5,x_6, x_7,x_8,x_9\}$.

\item $\sufc(\bQ_3) = \{\{(t,\text{hot})\}\} \cup \{ x_1, x_2, x_4, x_6,x_7\}$.
\end{itemize}    
Note that $\bQ_1$ has 8 GSRs; the first one reads: 
\textit{If the temperature had been mild, the decision would have been different from 
beach}. 
\end{exs}

The following representation theorem shows that only explainers satisfying Strong Validity produce (GSRs). 

\begin{theorem} 
An explainer $\bL$ satisfies Strong Validity \textbf{iff} $\bL \sqsubseteq \sufc$.		
\end{theorem}

Table~\ref{tab:axioms2} summarizes the behavior of $\sufc$ and any explainer $\bL\sqsubseteq\sufc$ 
with respect to all axioms. It confirms that the two forms of counterfactuals—necessary and 
sufficient—convey distinct types of information, as they satisfy different subsets of axioms. 
Unlike $\necp$, $\sufc$ guarantees at least one non-empty explanation for any query but 
violates Feasibility. 

\vspace{0.1cm}
\noindent \textbf{To sum up}, Theorem~\ref{th1-sufc} shows that the axiom of Strong Validity gives 
rise to a new type of global counterfactuals—those that highlight feature–value combinations whose 
presence guarantees the avoidance of a given class.
\subsection{Local Sufficient Reasons}

Let us now explain decisions of individual instances. We distinguish two 
types of explanations: \textit{sceptical} and \textit{credulous}. 

\subsubsection{Sceptical Sufficient Reasons (SSRs)}

They highlight features that are characteristic of instances assigned to classes different from that of the instance under consideration. As such, they act as distinguishing features of other classes and must be absent from the instance of interest. Consequently, SSRs consist exclusively of \textbf{new feature values} not already present in the instance.

\begin{definition}\label{sceptical-suf}
Let $\bQ = \langle \mbb{T}, \kappa,x\rangle$ be a query. 
A \emph{sceptical sufficient reason} (SSR) against $\kappa(x)$ is a set $E \in \mbb{E}(\mbb{T})$ s.t.: 
	\begin{itemize}	
        \item[]   
        $E \cap x = \emptyset$ \quad and \quad 
        $\forall y \in \mbb{F}(\mbb{T})$ s.t. $E \subseteq y$, $\kappa(y) \neq \kappa(x)$.
	\end{itemize}
$\ssuf$ denotes the explainer that generates \emph{all} SSRs.
\end{definition}

\begin{exs}{ex1bis}
Let $\bQ_i = \langle \mbb{T}, \kappa,x_i\rangle$, $i=1,2,3$. 
\begin{itemize}
	\item $\ssuf(\bQ_1) = \{\{(t,\text{mild})\}$, $\{(t,\text{freezing})\}\}$ $\cup$ $\{x_3, x_4, x_8, x_9\}$.
	
	\item $\ssuf(\bQ_2) =  \{\{(t,\text{hot})\}, \{(a,\text{reading})\}\} \cup \{x_3, x_6, x_7\}$. 
	
	\item $\ssuf(\bQ_3) = \{\{(t,\text{hot})\}\} \cup \{x_1, x_2, x_7\}$. 
\end{itemize}
Consider, for instance, the reason $\{(t,\text{mild})\}$. This feature–value pair characterizes instances 
labeled either mountain or cinema. There is no instance labeled beach with a mild temperature. 
\end{exs}

The next representation theorem establishes a one-to-one correspondence between 
the two axioms Novelty and Strong Validity, and the family of explainers that generate SSRs.

\begin{theorem}  
An explainer $\bL$ satisfies Novelty and Strong Validity \textbf{iff} $\bL \sqsubseteq \ssuf$. 
\end{theorem}

Table~\ref{tab:axioms2} shows that this fourth type of counterfactuals is 
different from local necessary reasons. However, both types may not exist. 

\vspace{0.1cm}
\noindent \textbf{To sum up}, the above characterization shows that Strong Validity and Novelty 
define a fourth type of counterfactual, based on absent feature–value combinations that always 
yield a different decision.


\subsubsection{Credulous Sufficient Reasons (CSRs).} 

Being the least restrictive type of reason, CSRs are feature–value combinations that distinguish an instance from \textbf{another instance} classified differently. Unlike SSRs, the presence of a CSR does not necessarily guarantee a change of decision. In the running example, \emph{skiing} is a CSR for instance $x_3$, since changing the value of $a$ to \emph{skiing} yields $x_4$, which receives a different decision. However, \emph{skiing} is not an SSR, as it also appears in $x_9$, whose class is the same as $x_3$.

\begin{definition}\label{credulous-suf}
Let $\bQ = \langle \mbb{T}, \kappa,x\rangle$ be a query. A \emph{credulous 
sufficient reason against} $\kappa(x)$ is $y\setminus x$ such that   
$y \in \mbb{F}(\mbb{T})$ and $\kappa(y) \neq \kappa(x)$.
$\csuf$ is the explainer that generates all CSRs. 
\end{definition}

\begin{exs}{ex1bis}
Let $\bQ_i = \langle \mbb{T}, \kappa,x_i\rangle$, $i=1,2,3$. 
\begin{itemize}
		\item $\csuf(\bQ_1) = \ssuf(\bQ_1)$.
		
		\item $\csuf(\bQ_2) =  \{\{(t,\text{hot})\}, \{(t,\text{freezing})\}$, $\{(a,\text{reading})\}$, \\ $\{(a,\text{skiing})\}\} \cup \{y \in \mbb{F}(\mbb{T}) 
		\mid y \cap x_2 = \emptyset \text{ and } \kappa(y) \neq \text{mountain}\}$.
		
		\item $\csuf(\bQ_3) = \{\{(t,\text{hot})\}, \{(a,\text{skiing})\}\} \cup \{y \in \mbb{F}(\mbb{T}) \mid y \cap x_3 = \emptyset \text{ and } \kappa(y) \neq \text{cinema}\}$.
\end{itemize}	
Note that $\{(a,\text{skiing})\}$ is not a SSR of $\bQ_3$ since skiing is not representative of 
mountain or beach. Indeed, $\kappa$ predicts cinema (class of $x_3$) for the instance $x_9 = ((t,\text{mild}), (a,\text{skiing}))$.
\end{exs}

In what follows, we present a representation theorem showing that the axioms Novelty and Weak Validity are satisfied exclusively by explainers that return CSRs.

\begin{theorem}\label{cara:csuf1}
An explainer $\bL$ satisfies Novelty and Weak Validity \textbf{iff} $\bL \sqsubseteq \csuf$.    
\end{theorem}

We now introduce the concept of a \textbf{faithful ranking}—a preorder over the set of partial assignments that strictly favors those introducing new feature values to an instance, where such revisions lead to a different decision. 

\begin{definition}
A \emph{faithful ranking} is a function $\bS$ that maps every query $\bQ = \langle \mbb{T}, \kappa, x \rangle$ to a preorder $\succeq_\bQ$ over $\mbb{E}(\mbb{T})$ such that $\forall E, E' \in \mbb{E}(\mbb{T})$, if: 
\begin{itemize}
\item $E \cap x = \emptyset$ and $\kappa(x_{\downarrow E}) \neq \kappa(x)$  and  
\item $E' \cap x \neq \emptyset$  or $\kappa(x_{\downarrow E'}) = \kappa(x)$,   
\end{itemize}
then $E \succ_\bQ E'$. 
\end{definition}

In what follows, we present a second representation theorem that connects the axioms Success, Novelty and Weak Validity with credulous sufficient reasons. It establishes an equivalence between these axioms and an explanation strategy grounded in preorders over partial assignments. Specifically, explainers satisfying these axioms impose a ranking over all partial assignments—from least to most preferred—and select the top-ranked ones as explanations. 

\begin{theorem}  \label{csuf-min}
	An explainer $\bL$ satisfies Success, Novelty and Weak Validity \textbf{iff} 
	there exists a faithful ranking $\bS$ mapping every query $\bQ = \langle\mbb{T},\kappa,x\rangle$ to a preorder 
	$\succeq_\bQ$ on $\mbb{E}(\mbb{T})$ such that:  
	\[\bL(\bQ) = \max(\mbb{E}(\mbb{T}), \succeq_\bQ).\]
\end{theorem}

The above result characterizes the sub-family of explainers $\bL \sqsubseteq \csuf$, identified in Theorem~\ref{cara:csuf1}, that guarantees the existence of at least one CSR. 
In a later section, we will show that this characterization, expressed in terms of faithful rankings, facilitates the comparison of existing explainers that generate CSRs.

\vspace{0.2cm}

The next result relates the three types of sufficient reasons (GSR, SSR, CSR) and shows that every local explanation is a subset of a global explanation.

\begin{proposition}\label{prop:csuf2}
Let $\bQ = \langle \mbb{T}, \kappa,x\rangle$ be a query. 
\begin{itemize}
\item $\ssuf \sqsubseteq \sufc$. 
	
\item $\ssuf \sqsubseteq \csuf$.
	 
\item For any $E \in \csuf(\bQ)$, $\exists E' \in \sufc(\bQ)$ s.t. $E = E' \setminus x$.
	
\item For any $E \in \sufc(\bQ)$, $E\setminus x \in \csuf(\bQ)$.	
\end{itemize}	
\end{proposition}

The properties summarized in Table~\ref{tab:axioms2} distinguish this fifth type of counterfactual from the other four. Unlike SSR, CRS are guaranteed to exist. Consequently, 
well-defined explainers can provide an outcome for every query.


\section{Human-Focused Assessment of Types}\label{comparison}

Cognitive science research shows that people generally prefer explanations that 
are short and concise \cite{Byrne19, WarrenDGK24, CassensHBWKKGBW21}. We examine 
how the five types of counterfactuals align with these needs, using the following criteria for comparison. 

\begin{description}
    \item [Existence:] Whether counterfactuals are guaranteed to exist.
    \item [Compactness:] The ability to convey multiple pieces of information in a concise way. 
    \item [Size:] The number of literals in a counterfactual.
    \item [Number:] The number of explanations provided for a query. 
    \item [Power of discrimination:] The ability of a counterfactual to discriminate between instances or classes.
\end{description}

Table~\ref{tab:axioms2} shows that necessary reasons and sceptical sufficient reasons may not always exist, 
whereas (global, credulous) sufficient reasons guarantee at least one explanation for any query. 
While the potential non-existence of reasons could be viewed as a drawback, \cite{Guidotti24} highlights that 
features in credulous sufficient reasons may not be \textit{actionable} (e.g., age cannot be changed). Therefore,  an 
explainer $\bL \sqsubseteq \csuf$ generating CSRs composed of features to be mutated may face cases where counterfactuals do not exist.  

When they do exist, necessary reasons (whether local or global) represent feature–value combinations whose 
absence alone causes a change in class, regardless of the new values assigned to these features. Hence, it is 
unnecessary to list all possible changes that could lead to a different decision. This allows necessary reasons 
to convey multiple pieces of information—the full range of changes—in a \textbf{compact} or concise way, making 
them particularly appealing for people. In contrast, sufficient reasons are not compact, as they enumerate all possible changes explicitly. 

Because of their compactness, necessary reasons are typically \textbf{fewer in number} than sufficient reasons. 
In the running example, $\necp$ returns a single explanation for query $\bQ_1$, whereas $\sufc$ yields eight GSRs. 
In fact, each alternative value of a feature occurring in a GNR gives rise to distinct GSRs. 
Global necessary reasons are also  \textbf{shorter} than global sufficient reasons, since each core literal 
provides a minimal explanation. 
 
Global counterfactuals are \textbf{discriminating}, as they capture feature–value pairs that distinguish 
one class from all the others. Likewise, sceptical (necessary, sufficient) reasons identify combinations that 
set an instance apart from others. CSRs, however, are less discriminating, as they may include pairs also 
found in instances with the same outcome. For instance, $(\text{t,freezing})$ is a CSR for $x_2$ but also 
belongs to $x_4$, which has the same decision as $x_2$.

\noindent \textbf{In conclusion}, necessary reasons, whenever they exist, seem more compatible with human preferences 
than sufficient reasons. 

\section{Related Work}\label{related-work}

Numerous studies have investigated counterfactual explanations (see \cite{Guidotti24} for a recent survey).
Some contributions propose new explainers, while others focus on identifying desirable properties that 
counterfactuals and the explainers that generate them should satisfy. 
In the following, we review both lines of work. 


\subsection{Counterfactuals Types and Explainers} 

Despite the rich body of work on counterfactual explanations, the vast majority has focused on \textbf{local} counterfactuals (e.g., \cite{marquis24,COOPER2023,Dandl20,DhurandharCLTTS18,KeaneKDS21,Ronny19,MothilalST20,barry22,Sahil21,wachter17}). Furthermore, given an input $x$, these works generate the "closest" instance $y$ that gets a different decision. 
The main distinction between the proposals lies in how closeness is defined. 
Some emphasize \textbf{sparsity}, seeking to minimize the set of features to be changed, while others prioritize \textbf{proximity}, allowing broader changes as long as the overall shift remains small according to a specified \textbf{distance metric} like \textit{Manhattan distance} ($L_1$ norm) or \textit{Gower distance} (see \cite{wachter17} for more examples). 
We show that all these explainers generate \textbf{credulous sufficient reasons}, thus are specific instances within the family of CSRs. In other words, each such explainer, say $\bL$, satisfies the condition: $\bL \sqsubseteq \csuf$. Furthermore, we characterize the faithful ranking defined by each explainer.

\vspace{0.1cm}
There are two explainers that promote sparsity. The first generates contrastive explanations~\cite{COOPER2023}—\textbf{minimal sets of features} that must be altered to get a different decision. For the sake of comparison, we present its corresponding explainer whose explanations are sets of literals (not features).

\begin{definition} 
We define $\bLwf$ as the explainer that, for any query 
$\bQ = \langle \mbb{T}, \kappa,x\rangle$ and $E \in \mbb{E}(\mbb{T})$, $E \in \bLwf(\bQ)$ iff:
\begin{itemize}
\item $\kappa(x_{\downarrow E}) \neq \kappa(x)$, 
\item $\nexists E' \in \mbb{E}(\mbb{T})$ s.t. $\kappa(x_{\downarrow E'}) \neq \kappa(x)$,  
$\feat(E') \subset \feat(E)$ (with $\feat(E)$ 
the set of \textbf{features} covered in $E$).
\end{itemize}
\end{definition}

\begin{exs}{ex1bis}
	Let $\bQ_i = \langle \mbb{T}, \kappa,x_i\rangle$, $i=1,2,3$. 
	\begin{itemize}
		\item $\bLwf(\bQ_1) = \{\{(t,\text{mild})\}, \{(t,\text{freezing})\}\}$,
		
		\item $\bLwf(\bQ_2) =  \{\{(t,\text{hot})\}, \{(a,\text{reading})\}\}$,
		
		\item $\bLwf(\bQ_3) = \{\{(t,\text{hot})\}\}$.
	\end{itemize}	  
\end{exs}

We define below the preorder corresponding to $\bLwf$.

\begin{definition}\label{succeq:f}
Let $\bQ = \langle\mbb{T},\kappa,x\rangle$ be a query and $\delta$ the weighting 
on $\mbb{E}(\mbb{T})$ such that for any $E \in \mbb{E}(\mbb{T})$:  
\begin{center}    
		$\delta(E) = \left\{\begin{array}{ll}
		  1    & \mbox{ if } E \cap x = \emptyset \mbox{ and } \kappa(x_{\downarrow E}) \neq \kappa(x)    \\
		+\infty    & \mbox{ otherwise }          \\
		\end{array}\right.$
\end{center} 
\noindent $\succeq^f_\bQ$ is a preorder on $\mbb{E}(\mbb{T})$ s.t. 
$\forall E, E' \in \mbb{E}(\mbb{T})$, $E \succeq^f_\bQ E'$ iff \\ 
$(\delta(E) = \delta(E') \ \mbox{ and } \ \feat(E) \subseteq \feat(E') ) \ \ \mbox{ or } \ \ \delta(E) < \delta(E')$         
\end{definition}

We show that $\bLwf$ is an instance of the family of CSRs.

\begin{theorem}\label{cara:wf}
The following properties hold:
\begin{itemize}
\item The function $\bS$ mapping every query $\bQ$ to $\succeq^f_\bQ$ on $\mbb{E}(\mbb{T})$ is a faithful ranking. 
\item For any query $\bQ$, $\bLwf(\bQ) = \max(\mbb{E}(\mbb{T}),\succeq^f_\bQ)$.
\end{itemize}
\end{theorem}

The second explainer, introduced in \cite{Guidotti24}, captures sparsity by selecting explanations of \textbf{minimal cardinality}.

\begin{definition}\label{lc}
We define $\bL_c$ as the explainer that, for any query 
$\bQ = \langle \mbb{T}, \kappa,x\rangle$ and $E \in \mbb{E}(\mbb{T})$,
$E \in \bL_c(\bQ)$ iff:  
\begin{itemize}
\item $\kappa(x_{\downarrow E}) \neq \kappa(x)$,    
\item $\nexists E' \in \mbb{E}(\mbb{T})$ s.t $\kappa(x_{\downarrow E'}){\neq}\kappa(x)$ and $|E'| < |E|$.
\end{itemize}
\end{definition}

\noindent \textbf{Remark:} It is easy to show that $\bL_c \sqsubseteq \bLwf$.

\begin{exs}{ex1bis}
For every $i = 1,2,3$, $\bL_c(\bQ_i) = \bLwf(\bQ_i)$.	   
\end{exs}

This explainer ranks CSRs according to their cardinality. 

\begin{definition}\label{succeq:c}
Let $\bQ = \langle\mbb{T},\kappa,x\rangle$ be a query and $\sigma$ a weighting on $\mbb{E}(\mbb{T})$ such that for any $E \in \mbb{E}(\mbb{T})$:  
\vspace{-0.1cm}
\begin{center}    
$\sigma(E) = \left\{\begin{array}{ll}
	|E|    & \mbox{ if } E \cap x = \emptyset \mbox{ and } \kappa(x_{\downarrow E}) \neq \kappa(x)    \\
+\infty    & \mbox{ otherwise }          \\
\end{array}\right.$
\end{center} 
We define $\succeq^c_\bQ$ as a preorder on $\mbb{E}(\mbb{T})$ such that:  
$$\forall E, E' \in \mbb{E}(\mbb{T}), \ \  E \succeq^c_\bQ E' \ \mbox{ iff } \ \sigma(E) \leq \sigma(E').$$ 
\end{definition}
	
We show that $\bL_c$ is an instance of the family of CSRs.

\begin{theorem}\label{lem2}
	The following properties hold:
	\begin{itemize}
		\item The function $\bS$ mapping every query $\bQ$ to $\succeq^c_\bQ$ on $\mbb{E}(\mbb{T})$ is a faithful ranking. 
		\item For any query $\bQ$, $\bL_c(\bQ) = \max(\mbb{E}(\mbb{T}), \succeq^c_\bQ)$.
	 \end{itemize}	
\end{theorem}

Existing approaches that compare candidate counterfactuals using a distance measure—including those proposed in \cite{marquis24,Dandl20,DhurandharCLTTS18,KeaneKDS21,Ronny19,MothilalST20,barry22,Sahil21, wachter17}—are \textbf{all} defined in the following manner.

\begin{definition}\label{ld}
Let $\bQ = \langle\mbb{T},\kappa,x\rangle$ be a query, $E \in \mbb{E}(\mbb{T})$, 
and $\dd$ a distance measure on $\mbb{F}(\mbb{T})$.  
An explainer based on $\dd$ is a function $\bL_\dd$ such that 
$E \in \bL_\dd(\bQ)$ iff:  
\begin{itemize}
\item $\kappa(x_{\downarrow E}) \neq \kappa(x)$,  
\item $\nexists E'  \in \mbb{E}(\mbb{T}), \kappa(x_{\downarrow E'})\neq\kappa(x)
      \land \dd(x_{\downarrow E'}, x) < \dd(x_{\downarrow E}, x)$
\end{itemize}
\end{definition}

Existing distance-based explainers differ in the choice of distance measure $\dd$, and may yield explanations that 
differ from those produced by $\bLwf$ and $\bL_c$, as illustrated below. 

\begin{exs}{ex1bis}
Let $y = ((t,\text{freezing}), (a,\text{climbing}))$.
Assume a distance measure $\dd$ st.    
$\dd(x_2,x_1) < \dd(y,x_1)$. Then, 
$\bL_\dd(\bQ_1) = \{\{(t,\text{mild})\}\}$.	  
\end{exs}

We define the preorders produced by these explainers. 

\begin{definition}\label{succeq:d}
	Let $\bQ = \langle \mbb{T}, \kappa,x\rangle$ be a query, $\dd$ a distance measure on $\mbb{F}(\mbb{T})$, $\gamma$ a weighting on 
    $\mbb{E}(\mbb{T})$ s.t. $\forall E \in \mbb{E}(\mbb{T})$, 
    \vspace{-0.1cm}
	\begin{center}    
		$\gamma(E) = \left\{\begin{array}{ll}
		\dd(x_{\downarrow E}, x)       & \mbox{ if } E \cap x = \emptyset  \land  \kappa(x_{\downarrow E}) \neq \kappa(x)   \\
		+\infty & \mbox{ otherwise }          \\
		\end{array}\right.$
	\end{center}
    We define $\succeq_{\bQ,\dd}$ as a preorder on $\mbb{E}(\mbb{T})$ such that:
	$$\forall E, E' \in \mbb{E}(\mbb{T}), \ E \succeq_{\bQ,\dd} E' \ \mbox{ iff } \ \gamma(E) \leq \gamma(E').$$  
\end{definition}

We show that $\bL_\dd$ is an instance of the family of CSRs.

\begin{theorem}\label{lem3}
Let $\dd$ be distance measure.
\begin{itemize}
\item The function $\bS$ mapping every query $\bQ$ to $\succeq_{\bQ,\dd}$ on $\mbb{E}(\mbb{T})$ is a faithful ranking. 
\item For any query $\bQ$, $\bL_\dd(\bQ) = \max(\mbb{E}(\mbb{T}), \succeq_{\bQ,\dd})$.
\end{itemize}	
\end{theorem}

In Lewisian counterfactuals (first proposed in the seminal paper \cite{Lewis73}, then in \cite{Woodward2003,Ventura23}),
the change to the set of variables must be small. This can be modeled
by the following function $\bL_{\dd,\tau}$, which is also an instance of the family CSRs (i.e., $\bL_{\dd,\tau} \sqsubseteq \csuf$).

\begin{definition}\label{ldtau}
Let $\bQ = \langle\mbb{T},\kappa,x\rangle$ be a query, $E \in \mbb{E}(\mbb{T})$, 
and $\dd$ a distance measure on $\mbb{F}(\mbb{T})$.  
A threshold explainer based on $(\dd,\tau)$  is a function $\bL_{\dd,\tau}$ such that 
$E \in \bL_{\dd,\tau}(\bQ)$ iff 
$\kappa(x_{\downarrow E}) \neq \kappa(x)$ and 
$\dd(x_{\downarrow E}, x) < \tau$.
\end{definition}

There are three notable works that addressed \textbf{necessary reasons}. 
The two, proposed by \cite{IgnatievNA020,LiuL23}, introduce functions that generate subset-minimal necessary reasons for instances while violating the axiom of sceptical validity. Their reasons are SNRs when core literals exist or when all features in the theory are binary. In the general case, however, the reasons merely guarantee that some change exists, without specifying which one. In this case, they are of no use to a human user. 
The third work, by \cite{Amgoud22}, studied 
global counterfactuals. It proposed some axioms, including Coreness, Success, and Non-Triviality, and defined two explainers, say $\bL_1$ and $\bL_2$: one that returns a single reason, namely the full set of core literals for the class assigned to the instance, and another that returns non-empty, subset-minimal necessary reasons. Both explainers yield GNRs, i.e., $\bL_1 \sqsubseteq \necp$ and $\bL_2 \sqsubseteq \necp$. 
Our work is more general. It explores a wider spectrum of counterfactual types—including GNR, SNR, GSR, SSR, and CSR—and contributes deeper theoretical results, most notably a series of representation theorems grounded in six novel axioms.

A notion of necessary reason was introduced in \cite{darwiche20} as the intersection of all partial assignments that guarantee the class of an instance. This reason is unique, may be empty, and its removal does not necessarily change the decision. As such, the corresponding explainer does not provide enough information for the user to understand how to alter the prediction, and thus does not qualify as a counterfactual explanation.


\subsection{Properties of Counterfactual (Explainers)}

Properties of explanations have been studied in several papers~\cite{BodriaGGNPR23,Guido19,VermaBHHDS24,WinikoffTR25,WarrenSK22} and various properties of couterfactuals 
have been proposed in \cite{LooverenK21,KanamoriTKI22,barry22,Tsiourvas2024,Verma21,KusnerLRS17,vonkgel22}. 
These papers focused on credulous sufficient reasons (CSRs) and discussed various quantitative 
properties that serve to define functions that select a subset of counterfactuals returned by $\csuf$.  
Examples of properties are 
\textit{Minimality, Similarity, Plausibility, Discriminative power, Actionability, Causality} and \textit{Fairness}. 
Minimality guarantees a minimal cardinality, similarity a minimal 
distance from the instance to be explained, whereas plausibility states that the new values in a CSR should be similar to those of instances seen during the training phase~\cite{LooverenK21,KanamoriTKI22,barry22,Tsiourvas2024,Verma21}. 
Discriminative power is a subjective property which states that a counterfactual should help in figuring out why a different outcome can be obtained. 
Actionability guarantees that the features are actionable, i.e., can be mutated. 
This property is of great importance and may lead to the violation of Success.  
Causality guarantees that counterfactuals preserve known causal relations between attributes~\cite{Verma21}. 
A counterfactual is fair if it does not refer to protected features~\cite{KusnerLRS17,vonkgel22}. Fairness of explanations is
intertwined with fairness of decisions; indeed, the former can be used to 
evaluate the latter~\cite{fairnessCP}.

Properties of sets of counterfactual explanations have also been proposed, 
namely \textit{Diversity} and \textit{Stability}, called also \textit{Robustness}~\cite{JiangL0T24}. 
Diversity states that chosen counterfactuals should be different, e.g., refer to 
different attributes~\cite{Leofante2024,MothilalST20,barry22,TsirtsisR20}. 
Stability~\cite{Salva19} posits that similar instances should receive similar explanations. 

\noindent \textbf{Remark:}  Sections~\ref{necessary-reasons} and \ref{sufficient-reasons} introduced five broad families of explainers, each producing a specific type of counterfactual (GNR, SNR, GSR, SSR, CSR). The properties presented above \textbf{complement} this classification by helping to identify ‘reasonable’ explainers within each family. For example, one may define an explainer that generates SNRs that are subset-minimal, actionable, diverse, and fair.

\vspace{0.1cm}

Finally, while \cite{BassanAK24} studied properties of explainers for \textbf{‘why’} questions and analyzed decision complexity for three \textbf{fixed} classifiers, our work focuses on \textbf{‘why not’} questions. We introduce novel axioms that characterize five families of counterfactual explainers and analyze explanation complexity under \textbf{arbitrary} classifiers.

\section{Computational Complexity}\label{complexity}

This section analyses two basic computational problems associated with a counterfactual explainer $\bL$:
\begin{itemize}
\item {\sc DecideExp}($\bL$) is the decision problem whose input is a query $\bQ$ and an
explanation $E$, and which determines whether $E \in \bL(\bQ)$.
\item {\sc FindExp}($\bL$) takes as input a query $\bQ$ and returns  
an element $E \in \bL(\bQ)$ (or ``none'' if none exists).
\end{itemize}

The complexity of these  problems depends on the family of possible classifiers $\kappa$. 
To be concrete, 
we assume that $\kappa$ is a known boolean formula. 
If complexity depends on domain size, we also consider larger domains.

\begin{theorem} \label{thm:dec}
	Consider query $\bQ = \langle \mbb{T}, \kappa, x\rangle$ where
	$\mbb{T}$ is composed of $n$ boolean features and $\kappa$ is a boolean formula.
	\begin{enumerate}
		\item The problem {\sc DecideExp}$(\bL)$ is $O(n)$ when $\bL$ is $\csuf$, $\wL$, 
              or $\bLwf$. 
		\item The problem {\sc DecideExp}$(\bL)$ is co-NP-complete when $\bL$ is 
		$\necp$, $\sufc$, 
		$\ssuf$, $\bLa$, $\bL_d$. 
	\end{enumerate}
	For domain sizes which are arbitrary, 
	{\sc DecideExp}$(\bL)$ is $O(n)$ when $\bL$ is $\csuf$ or $\bLwf$. 
	For queries over non-boolean domains, {\sc DecideExp}$(\bL)$ is co-NP-complete
	when $\bL$ is $\wL$.
\end{theorem}

We study the complexity of generating an explanation.

\begin{theorem} \label{thm:find}
	Consider query $\bQ = \langle \mbb{T}, \kappa, x\rangle$ where
	$\mbb{T}$ is composed of $n$ boolean features and $\kappa$ is a boolean formula.
	\begin{enumerate}
		\item The problem {\sc FindExp}$(\bL)$ is $O(n)$ when $\bL$ is $\ssuf$. 
		\item The problem {\sc FindExp}$(\bL)$ is NP-hard when $\bL$ is 
		$\necp$, $\csuf$, $\wL$, $\sufc$, $\bLa$, $\bL_d$, or $\bLwf$. 
		In each of these cases, excluding $\bLa$ and $\bL_d$,
		it can be solved by at most $n$ calls to a SAT oracle. 
	\end{enumerate}
  	For queries over non-boolean domains {\sc FindExp}$(\bL)$ is NP-hard for $\ssuf$.
\end{theorem}

Theorems \ref{thm:dec} and \ref{thm:find} provide upper bounds on complexity for white-box classifiers
expressible as boolean formulas. However, it should be pointed out that
for specific types of classifiers, such as decision trees or monotonic functions, for example,
certain explainability queries are 
polynomial-time~\cite{BassanAK24,COOPER2023,HuangIICA022,IzzaIM22,ICML21}.

\section{Conclusion}\label{Conclusion}

This paper introduces the first axiomatic framework for counterfactual explanations. We define a set of axioms and establish representation theorems that characterize five types of counterfactuals, three of which are novel: SNR, GSR, and SSR. These types fall into two fundamental forms—necessary and sufficient reasons. Necessary reasons are concise and informative but may not always exist, whereas sufficient reasons are guaranteed to exist but can be numerous, especially when feature domains are large. Both forms are generally intractable to compute, except in restricted settings such as Boolean feature spaces.

This work opens several avenues for future research. One direction is to characterize, for each type of counterfactual, the explainers that generate explanations of that type while satisfying established properties such as actionability, minimality, fairness, and diversity. Another direction is to study the robustness of the newly introduced types of counterfactuals.

\section*{Acknowledgments}

Our work was supported by the AI Interdisciplinary Institute ANITI, funded by the France 2030 program under the Grant agreement n°ANR-23-IACL-0002, and by
the French National Research Agency project ForML ANR-23-CE25-0009.

\bibliography{refsExplanations}

\appendix
\setcounter{prop}{0}
\setcounter{theorem}{0}
\setcounter{property}{0}



\section{Example}

We provide a new example that is useful for some of the proofs provided in this document.

\begin{example}\label{ex1}
Consider the theory described below and the outcomes of the classifier $\kappa_1$.

\begin{multicols}{2}
	\begin{tabular}{c|cc|c}\hline
	       $\mbb{F}(\mbb{T}_1)$   &$f_1$& $f_2$ & $\kappa_1(x_i)$   \\\hline
                        $x_1$    & 0   & 0     & $c_1$ \\
                        $x_2$    & 0   & 1     & $c_2$ \\          
				        $x_3$    & 1   & 0     & $c_2$ \\
				          $x_4$    & 1   & 1     & $c_3$ \\\hline
	\end{tabular}
\qquad 
{\small
\begin{itemize}
\item [] \hspace{-0.4cm}	 $\core(c_1) {=} \{(f_1,0), (f_2,0)\}$ 
\item [] \hspace{-0.4cm}	 $\core(c_2) {=} \emptyset$ 
\item [] \hspace{-0.4cm}	 $\core(c_3) {=}  \{(f_1,1), (f_2,1)\}$ 
\end{itemize}}	
\end{multicols} 

\end{example}

\section{Proofs of Properties}

Below are some properties that are used in some proofs of the main theorems.

\begin{property}
Let $\mbb{T} \in \mathtt{Th}$. $\forall x \in \mbb{F}(\mbb{T})$, 
$\forall E \in \mbb{E}(\mbb{T})$,  
$x_{\downarrow E} \in  \mbb{F}(\mbb{T})$.
\end{property}

\begin{proof}
Follows from the definition of $x_{\downarrow E}$ and the consistency of partial assignments, i.e. for any partial assignment $E$, $\nexists l, l' \in E$ such that $l = (f,v)$, $l' = (f',v')$ and $f = f'$ while $v \neq v'$. 
\end{proof}


\begin{property}\label{propy:core}
Let $\mbb{T} = \langle \F, \d, \C\rangle$ be a theory, $c \in \C$. 
$$\core(c) = \bigcap\limits_{(x \in \mbb{F}(\mbb{T})) \ \wedge \  (\kappa(x) = c)} x.$$  
\end{property}
\begin{proof}
    Straightforward from Definition~1.
\end{proof}


\begin{property}\label{propy:sufc}
Let $\bQ = \langle \mbb{T}, \kappa,x\rangle$ be a query.
$\forall y \in \mbb{F}(\mbb{T})$ such that $\kappa(y) \neq \kappa(x)$, $y \in \sufc(\bQ)$.
\end{property}
\begin{proof}
Follows from Definition~6.
\end{proof}


\begin{property}\label{propy:csuf}
Let $\bQ = \langle \mbb{T}, \kappa,x\rangle$ be a query. 
\begin{itemize}
\item $\csuf(\bQ) = \{E \in \mbb{E}(\mbb{T}) \mid x\cap E = \emptyset \mbox{ and } \kappa(x_{\downarrow E}) \neq \kappa(x)\}$, 
\item $\csuf(\bQ) \neq \emptyset$. 
\end{itemize}
\end{property}

\begin{proof}
Let $\bQ = \langle \mbb{T}, \kappa,x\rangle$ be a query and $E \in \mbb{E}(\mbb{T})$. 
\begin{itemize}
\item Assume $E \in \csuf(\bQ)$, so $E = y\setminus x$ such that $\kappa(y) \neq \kappa(x)$. So, $E \cap x = \emptyset$ and $x_{\downarrow E} = y$. 
\item Assume now that $x \cap E = \emptyset$ and $\kappa(x_{\downarrow E}) \neq \kappa(x)$. Let $y = x_{\downarrow E}$. It follows that $E = y\setminus x$ and so $E \in \csuf(\bQ)$.

\item We show that $\csuf(\bQ) \neq \emptyset$. Since $\kappa$ is surjective and $|\C| > 1$, we have $\exists y \in \mbb{F}(\mbb{T})$ such that 
$\kappa(y) \neq \kappa(x)$, so $y\setminus x \in \csuf(\bQ)$.
\end{itemize}
\end{proof}

\section{Proofs of Propositions}

\begin{proposition}
The following implications hold. 
\begin{itemize}
\item Coreness $\Rightarrow$ Feasibility. 
\item Coreness and Non-Triviality $\Rightarrow$ Sceptical Validity.
\item Sceptical Validity $\Rightarrow$ Non-Triviality.
\item Weak Validity $\Rightarrow$ Non-Triviality.  
\item Strong Validity $\Rightarrow$ Weak Validity. 
\item Novelty and Non-Triviality $\Rightarrow$ Sceptical Validity.
\end{itemize}
\end{proposition}

\begin{proof}{}
Let $\bL$ be an explainer, $\bQ = \langle \mbb{T}, \kappa,x\rangle$ a query and $E \in \bL(\bQ)$. 
\begin{itemize} 
\item Let $\bL$ satisfy Coreness. 
So, $E \subseteq \core(\kappa(x))$. From Property~\ref{propy:core}, 
$E \subseteq \bigcap\limits_{(y \in \mbb{F}(\mbb{T})) \ \wedge \  (\kappa(y) = \kappa(x))} y,$ so $E \subseteq x$, which shows that $\bL$ satisfies Feasibility.

\item Assume that $\bL$ satisfies Coreness and  Non-Triviality.  
From Non-Triviality, $E \neq \emptyset$. 
From Coreness, $E \subseteq \core(\kappa(x))$. From Property~\ref{propy:core}, 
$E \subseteq \bigcap\limits_{(y \in \mbb{F}(\mbb{T})) \ \wedge \  (\kappa(y) = \kappa(x))} y.$
Assume some $z \in \mbb{F}(\mbb{T})$ such that $x \setminus z = E$,  
then $E \not\subseteq z$. 
So, $\bigcap\limits_{(y \in \mbb{F}(\mbb{T})) \ \wedge \  (\kappa(y) = \kappa(x))} y \not\subseteq z$, hence $\kappa(z)\neq\kappa(x)$, 
which shows Sceptical Validity.

\item Assume $\bL$ satisfies Sceptical Validity. Then, $\forall \bQ = \langle \mbb{T}, \kappa, x\rangle$, $\forall E \in \bL(\bQ)$, $\forall y \in x\circleddash E$, 
$\kappa(y) \neq \kappa(x)$. Obviously, $\emptyset \notin \bL(\bQ)$ since 
otherwise we get $\kappa(x) \neq \kappa(x)$. This shows that $\bL$ satisfies Non-Triviality.

\item Assume $\bL$ satisfies Weak Validity. Then, $\forall \bQ = \langle \mbb{T}, \kappa, x\rangle$, $\forall E \in \bL(\bQ)$, $\kappa(x_{\downarrow E}) \neq \kappa(x)$. Obviously, $\emptyset \notin \bL(\bQ)$ since 
otherwise we get $\kappa(x) \neq \kappa(x)$. Hence $\bL$ satisfies Non-Triviality.

\item Suppose that $\bL$ satisfies Strong Validity.
Let $E \in \mbb{E}(\mbb{T})$ and $y = x_{\downarrow E}$. So, 
      $E \subseteq y$. From Strong Validity, $\kappa(y) \neq \kappa(x)$. So, $\bL$ satisfies Weak Validity. 

\item Assume $\bL$ satisfies Novelty and Non-Triviality. Assume also that $\bL$ violates 
Sceptical Validity. Then, $\exists E \in \bL(\bQ)$, $\exists y \in x\circleddash E$ such that $\kappa(y) = \kappa(x)$. So, $E \subseteq x$ (C). From Non-Triviality, $E \neq \emptyset$ and from Novelty, $E \cap x = \emptyset$, which contradicts (C). Thus, 
$\bL$ satisfies Sceptical Validity. 

\end{itemize} 
\end{proof}

\begin{proposition}\label{nec-links1}
It holds that $\necp \sqsubseteq \wL$.
\end{proposition}

\begin{proof}{}
Let $\bQ = \langle \mbb{T}, \kappa,x\rangle$ be a query s.t. $\kappa(x) = c$. 
We show that $\necp(\bQ) \subseteq \wL(\bQ)$. Let $E \in \necp(\bQ)$. 
From Theorem~\ref{th1-cara}, $E \subseteq \core(c)$. From Property~\ref{propy:core}, $E \subseteq \bigcap\limits_{\kappa(y) = \kappa(x)}y$. 
Furthermore, $E$ is non-empty.
For any $y \in x \circleddash E$, $x\setminus y = E$, thus 
$E \cap y = \emptyset$ and so $\kappa(y) \neq \kappa(x)$, which means that $E \in \wL(\bQ)$. 
\end{proof}


\begin{proposition}
	Let $\bQ = \langle \mbb{T}, \kappa,x\rangle$ be a query. 
	\begin{itemize}
		\item $\ssuf \sqsubseteq \sufc$. 
		
		\item $\ssuf \sqsubseteq \csuf$.
		
		\item For any $E \in \csuf(\bQ)$, $\exists E' \in \sufc(\bQ)$ s.t. $E = E' \setminus x$.
		
		\item For any $E \in \sufc(\bQ)$, $E\setminus x \in \csuf(\bQ)$.	
	\end{itemize}	
\end{proposition}

\begin{proof}
	Let $\bQ = \langle \mbb{T}, \kappa,x\rangle$ be a query. 
	
	\begin{itemize}
		\item The inclusion $\ssuf \sqsubseteq \sufc$ is straightforward. 
		
		\item Let $E \in \ssuf(\bQ)$. So, $E \cap x = \emptyset$ and $\forall y \in \mbb{F}(\mbb{T})$ such that $E \subseteq y$, $\kappa(y) \neq \kappa(x)$. 
		Since $E \in  \mbb{E}(\mbb{T})$, from Property~\ref{propy:instances}, 
		$x_{\downarrow E} \in \mbb{F}(\mbb{T})$. Let $z = x_{\downarrow E}$. From $E \cap x = \emptyset$, 
		$E = z\setminus x$. Since $E  \subseteq z$, we have $\kappa(z) \neq \kappa(x)$. So $E \in \csuf(\bQ)$.

		\item Let $E \in \csuf(\bQ)$. So, $\kappa(x_{\downarrow E}) \neq \kappa(x)$. 
		Let $y = x_{\downarrow E}$. 
		From Property~\ref{propy:instances}, $y \in \mbb{F}(\mbb{T})$.  
		From Property~\ref{propy:sufc}, $y \in  \sufc(\bQ)$. 
		Since $E \cap x = \emptyset$, $E = y\setminus x$.
		
		\item Let $E \in \sufc(\bQ)$. Let $E' = E\setminus x$ and $y = x_{\downarrow E'}$. 
		Since $E \subseteq y$, $\kappa(y) \neq \kappa(x)$, so $E' \in \csuf(\bQ)$. 
	\end{itemize}
\end{proof}


\begin{proposition}\label{prop:core}
Let $\bQ = \langle \mbb{T}, \kappa,x\rangle$ be a query. 

\begin{itemize}
\item For any $E \in \necp(\bQ)$, 
      \begin{itemize}
         \item $E \subseteq x$, 
         \item $E \subseteq \core(\kappa(x))$.  
      \end{itemize}
\item For any $E \subseteq \core(\kappa(x))$ s.t. $E \neq \emptyset$, $E \in \necp(\bQ)$.
\end{itemize}
\end{proposition}

\begin{proof}{}
Let $\bQ = \langle \mbb{T}, \kappa,x\rangle$ be a question. 
\begin{itemize}

\item Assume that $E \in \necp(\bQ)$. 
\begin{itemize}
\item From Definition~4, 
$E \nsubseteq x$ would imply the contradiction
$\kappa(x) \neq \kappa(x)$.

\item By Definition~4,  $\forall y \in \mbb{F}(\mbb{T})$ s.t. 
$E \nsubseteq y$, $\kappa(y) \neq \kappa(x)$. Hence, $\forall l \in E$, $l \in \core(\kappa(x))$ and 
so  $E \subseteq \core(\kappa(x))$. 
\end{itemize}

\item Let now $E \subseteq \core(\kappa(x))$. 
Assume now that $E \neq \emptyset$.
From Property~\ref{propy:core}, 
$\forall y \in \mbb{F}(\mbb{T})$ such that $\kappa(y) = \kappa(x)$, $E \subseteq y$ (1).
Assume now that $E \notin \necp(\bQ)$. So, 
$\exists y \in \mbb{F}(\mbb{T})$ s.t. $E \not\subseteq y$ and $\kappa(y) = \kappa(x)$, 
which contradicts (1). Hence $E \in \necp(\bQ)$.  
\end{itemize}
\end{proof}

\begin{proposition}\label{prop:local-nec}
Let $\bQ = \langle \mbb{T}, \kappa,x\rangle$ be a query.  
For any $E \subseteq \core(\kappa(x))$ such that $E \neq \emptyset$,   $E \in \wL(\bQ)$.
\end{proposition}

\begin{proof}
Let $\bQ = \langle \mbb{T}, \kappa,x\rangle$ be a query, 
$E \subseteq \core(\kappa(x))$ s.t. $E \neq \emptyset$.   
      Assume that $E \notin \wL(\bQ)$. 
      From Property~\ref{propy:core}, $E \subseteq x$. 
      So, $\exists y \in \mbb{F}(\mbb{T})$ such that $x \setminus y = E$  
      and $\kappa(y) = \kappa(x)$. 
      From Property~\ref{propy:core}, $E \subseteq y$, which 
      is impossible if $x \setminus y = E$ and $E \neq \emptyset$.
\end{proof}


\section{Proofs of Theorems}


\begin{theorem}
The axioms of every set $(\mathtt{I}_i)$ are incompatible.
\begin{itemize} 
    \item $(\mathtt{I}_1)$ Success, Non-Triviality and Coreness. 
    \item $(\mathtt{I}_2)$ Success, Feasibility and Sceptical Validity.
    \item $(\mathtt{I}_3)$ Success, Novelty and Strong Validity.  
    \item $(\mathtt{I}_4)$ Success, Non-Triviality, Feasibility and Novelty. 
    \item $(\mathtt{I}_5)$ Success, Feasibility and Weak Validity.  
    \item $(\mathtt{I}_6)$ Success, Non-triviality, Equivalence and Feasibility.
    \item $(\mathtt{I}_7)$ Success, Non-triviality, Equivalence and Novelty.
\end{itemize}
\end{theorem}

\begin{proof} 
Let $\mathtt{I}_i$ be such that $i\in \{1, ...,7\}$. 
\begin{itemize}
\item[$\mathtt{I}_1$] 
    Let $\bL$ be an explainer and $\bQ = \langle \mbb{T}, \kappa, x\rangle$ a query such that $\core(\kappa(x)) = \emptyset$. 
    
    Assume that $\bL$ satisfies Coreness and Success. 
    From Success, $\bL(\bQ) \neq \emptyset$. Hence, 
    $\exists E \in \bL(\bQ)$. From Coreness, $E \subseteq \core(\kappa(x))$ and so $E = \emptyset$, 
    which means $L$ violates Non-Triviality.

\item[$\mathtt{I}_2$] Consider a theory made of two features $a, b$ such that 
    $\d(a) = \{0,1\}$ and $\d(b) = \{0,1,2\}$. Consider the binary classifier $\kappa$ such that $\kappa(x) = 1$ iff 
    $x = ((a,0), (b,1))$. Let us consider the instance 
    $z = ((a,0), (b,0))$ (with $\kappa(z) = 0$) 
    and the corresponding query $\bQ$. 
    Assume an explainer $\bL$ that satisfies Feasibility, Sceptical Validity and Success. 
    From Feasibility, $\bL(\bQ) \subseteq \mbb{P}(z)$ (where $\mbb{P}(z)$ is the powerset of $z$). 
    From Success, $\exists E \in \bL(\bQ)$. 
    It is easy to check that,
    whatever the value of $E$,
    $\exists y \in z\circleddash E$ such that  
    $\kappa(y) = \kappa(z)$, which contradicts satisfaction of Sceptical Validity.

\item[$\mathtt{I}_3$] Assume that $\bL$ satisfies Success, Novelty and Strong Validity. 
     Consider the same theory as above in the case $\mathtt{I}_2$ but this time
     consider the instance $w=((a,1),(b,1))$. By Success, $\bL(\bQ)$
     is non-empty. The only explanation $E$ satisfying Strong Validity
     is $\{(a,0),(b,1)\}$, but this explanation violates Novelty.

\item[$\mathtt{I}_4$] Assume $\bL$ satisfies Success and Non-Triviality, then $\exists E \in \bL(\bQ)$ such that $E \neq \emptyset$. If $\bL$ satisfies Novelty, then $x \cap E = \emptyset$, which means $\bL$ violates Feasibility. 

\item [$\mathtt{I}_5$] Assume $\bL$ satisfies Success and Feasibility. 
Then, $\exists E \in \bL(\bQ)$ such that $E \subseteq x$. 
So, $x_{\downarrow E} = x$ and $\kappa(x_{\downarrow E}) = \kappa(x)$, which violates Weak Validity.

\item [$\mathtt{I}_6$] Let $\bQ = \langle \mbb{T}, \kappa, x\rangle$ and 
$\bQ' = \langle \mbb{T}, \kappa, y\rangle$ be queries such that $\kappa(x) = \kappa(y)$ and $x \cap y = \emptyset$. 
Assume $\bL$ satisfies Success, Non-triviality, Equivalence and Feasibility. 
From Success, $\bL(\bQ) \neq \emptyset$ and $\bL(\bQ') \neq \emptyset$. 
From Non-triviality,  $\emptyset \notin \bL(\bQ)$ and $\emptyset \notin \bL(\bQ')$. 
So, $\exists E \in \bL(\bQ)$, $\exists E' \in \bL(\bQ')$ such that 
$E \neq \emptyset$ and $E' \neq \emptyset$. 
From Feasibility, $E \subseteq x$ and $E' \subseteq y$. Hence, 
$E \notin \bL(\bQ')$ and $E' \notin \bL(\bQ)$. 
So, $\bL(\bQ) \neq \bL(\bQ')$ while 
Equivalence implies $\bL(\bQ) = \bL(\bQ')$.

\item [$\mathtt{I}_7$] Assume $\bL$ satisfies 
Success, Non-triviality, Equivalence and Novelty.
Let $\bQ = \langle \mbb{T}, \kappa, x\rangle$ and 
$\bQ' = \langle \mbb{T}, \kappa, y\rangle$ be queries such that 
$\mbb{T}$ is made of two binary features $f_1$ and $f_2$, and $\kappa$ is a binary classifier such that $\kappa(x) = f_1 \equiv f_2$.     
Let $x = \langle(f_1, 0), (f_2, 0)\rangle$ and 
$y = \langle(f_1, 1), (f_2, 1)\rangle$. Then, 
$\kappa(x) = \kappa(y) = 1$ and $x \cap y = \emptyset$. 
From Equivalence, $\bL(\bQ) = \bL(\bQ')$. 
From Success, $\exists E \in \bL(\bQ)$, and so $E \in \bL(\bQ')$.
From Novelty, $E \cap x = \emptyset$ and $E \cap y = \emptyset$, which 
imply $E = \emptyset$ (since $x$ and $y$ cover all possible assignments to the two
variables of $\mbb{T}$). But then non-triviality cannot hold.

\end{itemize}    
\end{proof}


\begin{theorem} 
All combinations of axioms not disallowed by Proposition~1 and Theorem~1 are satisfied by some explainer.
\end{theorem}

\begin{proof}
We consider the explainers $\csuf$ and $\sufc$, together with the following three explainers:
\begin{itemize}
\item $\bL_0$ defined by $\forall \bQ$, $\bL_0(\bQ) = \emptyset$.
\item $\bL_1$ defined by $\forall \bQ$, $\bL_1(\bQ) = \{\emptyset\}$.
\item $\bL_2$ defined by $\forall \bQ = \langle \mbb{T}, \kappa, x, \D \rangle$,
$\bL_2(\bQ) = \{ x \setminus y \mid \kappa(y) \neq \kappa(x)\}$.
\end{itemize}
Clearly, $\bL_0$ and $\bL_1$ are not useful explainers since, being constant functions, they effectively provide no information.
The explainer $\bL_2$ can be seen as a less informative version of $\csuf$: for each $y \in \mbb{F}$ such that
$\kappa(y) \neq \kappa(x)$, $\csuf$ returns $y \setminus x$, the new values of the features to be changed,
whereas $\bL_2$ returns the less informative $x \setminus y$, the old values of the features to be changed.

$\bL_0$ trivially satisfies all axioms apart from Success. Thus we only need to consider sets of axioms including Success.

From Proposition~1, if an explainer does not satisfy Non-Triviality, then it cannot satisfy any of Sceptical
Validity, Weak Validity or Strong Validity. It is easy to verify that $\bL_1$ satisfies the five other axioms.
Thus we only need to consider sets of axioms including Success and Non-Triviality.

If an explainer satisfies neither Feasibility nor Novelty, then it cannot satisfy Coreness (by Proposition~1).
From Table~2, we can see that the explainer $\sufc$ satisfies the six other axioms. Thus we only need to consider
sets of axioms that satisfy either Feasibility or Novelty (along with Success and Non-Triviality).

If an explainer satisfies Success, Non-Triviality and Feasibility, then by Theorem~1 and Proposition~1,
 it cannot satisfy any of the other six axioms. The axioms Success, Non-Triviality and Feasibility are
compatible since satisfied by $\bL_2$.

Finally, if an explainer satisfies Success, Non-Triviality and Novelty, then by Proposition~1 and
Theorem~1, it cannot satisfy any of Coreness, Feasibility, Equivalence or Strong Validity.
Since the other five axioms are satisfied by $\csuf$ (see Table~2), this completes the proof.
\end{proof}

\begin{theorem}\label{th1-cara}
An explainer $\bL$ satisfies Coreness and Non-Triviality\  \textbf{iff} $\ \bL \sqsubseteq \necp$.	
\end{theorem}

\begin{proof}{}
Let $\bQ = \langle \mbb{T}, \kappa,x\rangle$ be a query. 

\begin{itemize}
\item Let $\bL$ be an explainer that satisfies Coreness and Non-Triviality. 
So, for any $E \in \bL(\bQ)$, $E \subseteq \core(\kappa(x))$ and $E \neq \emptyset$.  
From Property~\ref{propy:core}, $E \subseteq \bigcap\limits_{(y \in \mbb{F}(\mbb{T})) \ \wedge \  (\kappa(y) = \kappa(x))} y$.
So, $\forall y \in \mbb{F}(\mbb{T})$ s.t. $E \not\subseteq y$, $\kappa(y) \neq \kappa(x)$. 
Thus, $E \in \necp(\bQ)$. 

\item Assume now an explainer $\bL$ such that $\bL \sqsubseteq \necp$. 
If for any query $\bQ$, $\bL(\bQ) = \emptyset$, then $\bL$ satisfies Coreness and Non-Triviality in a trivial way. 
Assume now that $E \in \bL(\bQ)$. Then, $E \in \necp(\bQ)$
and so $E \neq \emptyset$. From the first item of Proposition~\ref{prop:core}, $E \subseteq \core(\kappa(x))$. So, $\bL$ satisfies Coreness and Non-Triviality.  
\end{itemize}

\end{proof}


\begin{theorem}\label{th5-cara}
The properties of Table~\ref{tab:axioms2} hold. 
\end{theorem}

\begin{proof}{}
Let $\bL$ be an explainer and $\bQ = \langle \mbb{T}, \kappa, x\rangle$ be a query. 


\paragraph{Functions $\bL \sqsubseteq \necp$.}	

\begin{itemize}
\item \textit{Non-Triviality and Coreness.} Follow from Theorem~\ref{th1-cara}. 

\item \textit{Success.} From Theorem~\ref{incompatibility}, since $\bL$ satisfies 
Coreness and Non-Triviality, it violates Success.

\item \textit{Equivalence.} $\necp$ satisfies this axiom, but a
function $L \sqsubseteq \necp$ could be such that for two
instances $x,y$ such that $\kappa(x)=\kappa(y)$, $\bL(x) \neq 
\bL(y)$.

\item \textit{Feasibility.} From Proposition~\ref{prop:core}, $\forall E \in \bL(\bQ)$, $E \subseteq x$, so Feasibility is satisfied. 

\item \textit{Sceptical Validity.} Follows from Coreness and Non-Triviality (according to Proposition~\ref{links}).

 \item \textit{Novelty.} The explainer $\bL$ such that for any query $\bQ$, 
 $\bL(\bQ) = \emptyset$ satisfies Novelty. However, $\necp$ violates the axiom.

\item \textit{Strong Validity and Weak Validity.}  An explainer $\bL$ such that for any query $\bQ$, $\bL(\bQ) = \emptyset$ satisfies Strong Validity and 
Weak Validity in a trivial way. However, the two axioms are violated by $\necp$. 

\end{itemize}


\paragraph{Function $\bL = \necp$.}	

\begin{itemize}
\item \textit{Coreness and Non-Triviality.} From Theorem~\ref{th1-cara}. 

\item \textit{Success.} From Theorem~\ref{incompatibility}, since $\bL$ satisfies 
      Coreness and Non-Triviality, it violates Success.

\item \textit{Equivalence.} Follows from Proposition~\ref{prop:core}, namely the fact that 
for any $\bQ$, $\necp(\bQ)  = \{E \subseteq \core(\kappa(x)) \mid E \neq \emptyset\}$. 
 
\item \textit{Feasibility and Sceptical Validity.} Follow from the fact that any $\bL \sqsubseteq \necp$ satisfies Feasibility and Sceptical Validity.

\item \textit{Novelty and Weak Validity.} Consider a query $\bQ = \langle\mbb{T},\kappa,x\rangle$ such that $\core(\kappa(x)) \neq \emptyset$. Then $\exists E \in \necp(\bQ)$ such that $E \neq \emptyset$ and $E \subseteq \core(\kappa(x))$. From Property~\ref{propy:core}, $E \subseteq x$.

\begin{itemize}
\item Since $x \neq \emptyset$, then $x \cap E \neq \emptyset$ which contradicts Novelty.

\item Since $E \subseteq x$, then $x_{\downarrow E} = x$. 
From Weak Validity, $\kappa(x_{\downarrow E}) \neq \kappa(x)$, which is impossible. 
\end{itemize}
 
\item \textit{Strong Validity.}  Since $\necp$ violates Weak Validity, then from 
Proposition~\ref{links}, namely the implication Strong Validity $\Rightarrow$ Weak Validity, $\necp$ violates Strong Validity.

\end{itemize}
\paragraph{Functions $\bL \sqsubseteq \wL$.}

\begin{itemize}
	\item \textit{Feasibility and Sceptical Validity.} From Theorem~\ref{local:th:cara1}.

 
    \item \textit{Non-Triviality.} 
    From Proposition~\ref{links},
    Non-Triviality follows from Sceptical Validity.
	
	\item \textit{Success.} From Theorem~\ref{incompatibility}, the three 
 axioms Success, Sceptical Validity and Feasibility are incompatible. Thus, $\bL$ violates Success.
	
	\item \textit{Equivalence, Coreness, Novelty, Strong Validity, Weak Validity.} 
    The explainer $\bL$ such that for every query $\bQ$, 
    $\bL(\bQ) = \emptyset$ satisfies all these axioms while $\wL$ violates them. 

\end{itemize}

\paragraph{Functions $\bL = \wL$.}

\begin{itemize}
\item \textit{Success.} Follows from the fact that any $\bL \sqsubseteq \wL$ violates the axiom.

\item \textit{Feasibility, Non-Triviality, Sceptical Validity.}  Follow from the 
fact that any $\bL \sqsubseteq \wL$ satisfies these axioms.

\item \textit{Equivalence.} Consider a theory made of two binary features 
$f_1, f_2$ and the binary classifier $\kappa$ such that $\kappa(x) = f_1 \vee f_2$. Let $x = ((f_1,0), (f_2,1))$ and $y = ((f_1,1), (f_2,0))$. 
Obviously, $\kappa(x) = \kappa(y)$. 
Note that $\{((f_2,1))\}$ is a sceptical necessary reason for $\kappa(x)$ while 
it is not a reason for $\kappa(y)$ since $(f_2,1) \notin y$.

\item \textit{Coreness, Novelty, Strong/Weak Validity.} Consider the theory and the classifier below. 
	\begin{example}\label{ex2}
		Let $\bQ_1$ be the query concerning $x_1$. 
			\begin{multicols}{2}
				\begin{tabular}{c|cc|c}\hline
					$\mbb{F}(\mbb{T}_2)$   &$f_1$& $f_2$ & $\kappa_2(x_i)$   \\\hline
					\rowcolor{orange!10}            $x_1$    & 0   & 0     & $c_1$ \\
					$x_2$    & 0   & 1     & $c_1$ \\          
					$x_3$    & 1   & 0     & $c_2$ \\
					$x_4$    & 1   & 1     & $c_1$ \\\hline
				\end{tabular}
				\qquad 
				{\small
					\begin{itemize}
						\item [] \hspace{-0.4cm}	 $E_1 {=} \emptyset$ 
						\item [] \hspace{-0.4cm}	 $E_2 {=} \{(f_1,0)\}$ 
						\item [] \hspace{-0.4cm}	 $E_3 {=} \{(f_2,0)\}$ 
						\item [] \hspace{-0.4cm}	 $E_4 {=}  x_1$ 
				\end{itemize}}	
			\end{multicols} 
		Note that $\wL(\bQ_1) = \{E_2\}$ while 
		$(f_1,0) \notin \core(c_1)$, which violates Coreness. 
        Moreover, $E_2 \cap x_1 \neq \emptyset$ which violates Novelty, 
        $E_2 \subseteq x_2$ while $\kappa(x_1) = \kappa(x_2)$ which violates Strong Validity, and finally, $\kappa(x_{1\downarrow {E_2}}) = \kappa(x_1)$ which violates Weak Validity.  \end{example}	
\end{itemize}


\paragraph{Functions $\bL \sqsubseteq \sufc$.}

\begin{itemize}
\item \textit{Success.} The axiom is violated by the function 
$\bL$ such that for every query $\bQ$, $\bL(\bQ) = \emptyset$ and 
is satisfied by $\sufc$.

\item \textit{Non-Triviality.} Assume that $\emptyset \in \bL(\bQ)$. So, 
$\forall y \in \mbb{F}(\mbb{T})$, $\kappa(y) \neq \kappa(x)$, 
which is impossible because of $x$. So, $\bL$ satisfies the axiom. 


\item \textit{Feasibility, Coreness,  Novelty.} 
The axioms are satisfied by the function $\bL$ such that for every query $\bQ$, $\bL(\bQ) = \emptyset$ and violated by $\sufc$. 

\item \textit{Sceptical Validity.} There are two cases: 
\begin{itemize}
\item For every query $\bQ$, $\bL(\bQ) = \emptyset$. The axiom is trivially satisfied by $\bL$. 
\item For some query $\bQ$, $\bL(\bQ) \neq \emptyset$. 
From Non-Triviality, $\emptyset \notin \bL(\bQ)$. So, $\exists E \in \bL(\bQ)$, 
$E \neq \emptyset$. Moreover, $E \not\subseteq x$ \textbf{(C)}, 
otherwise by Definition~8, $\kappa(x) \neq \kappa(x)$. 
Recall that $x\circleddash E = \{y \in \mbb{F}(\mbb{T})\mid x\setminus y= E\}$, 
so $x\circleddash E = \emptyset$ (due to \textbf{(C)}). 
Then, $\bL$ satisfies Sceptical Validity in a trivial way.
\end{itemize}

\item \textit{Strong Validity.} Follows from Theorem~\ref{th1-sufc}.

\item \textit{Weak Validity.} Follows from Strong Validity according to Proposition~\ref{links}.

\end{itemize}
\paragraph{Function $\bL = \sufc$.}

\begin{itemize}
\item \textit{Success.} Follows from Property~\ref{propy:sufc} and 
the fact that $\kappa$ is surjective and a theory contains more than one class. 

\item \textit{Non-Triviality.} It follows from the fact that every function 
$\bL \sqsubseteq \sufc$ satisfies the axiom. 

\item \textit{Equivalence.} Follows from the definition of $\sufc$. 

\item \textit{Feasibility.} Since $\sufc$ satisfies Success, then for every 
query $\bQ$, $\sufc(\bQ) \neq \emptyset$. Let $E \in \sufc(\bQ)$. Assume 
that $E \subseteq x$. From Definition~8, we have the contradiction $\kappa(x) \neq \kappa(x)$.  

\item \textit{Coreness.} It follows from the fact that $\sufc$ satisfies Success and Non-Triviality, which are incompatible with Coreness (see Theorem~\ref{incompatibility}).  

\item \textit{Sceptical Validity.} Follows from the fact that every $\bL \sqsubseteq \sufc$ satisfies this axiom. 

\item \textit{Novelty.} From Example~\ref{ex1}, $x_2 \in \sufc(\bQ_1)$ while 
      $x_2 \cap x_1 \neq \emptyset$. 

\item \textit{Strong Validity.} From Theorem~\ref{th1-sufc}.

\item \textit{Weak Validity.} From Strong Validity and Proposition~\ref{links}.


\end{itemize}

\paragraph{Functions $\bL \sqsubseteq \ssuf, \ssuf, \bL \sqsubseteq \csuf$, and $\csuf$.}

	\begin{itemize}
	\item \textit{Success.} Consider the query $\bQ_2$ concerning the instance $x_2$ in Example~\ref{ex1}. Note that 
    $\ssuf(\bQ_2) = \emptyset$, So $\ssuf$ and 
	any $\bL \sqsubseteq \ssuf$ violate Success.  
 
	We show now that $\csuf$ satisfies Success. Since $|\C| > 1$ and $\kappa$ is surjective, 
	then $\exists y \in \mbb{F}(\mbb{T})$ such that $\kappa(y) \neq \kappa(x)$ and so  
	$y\setminus x$ is a credulous reason against $\kappa(x)$.
	
	Let $\bL \sqsubseteq \csuf$. If $\bL$ returns the empty set to every query, then 
    it violates Success. However, if $\bL = \csuf$, then it satisfies the axiom. 
	
	\item \textit{Non-Triviality.} Note that $\emptyset \notin \csuf(\bQ)$ since 
     $\kappa(x_{\downarrow \emptyset}) = \kappa(x)$. So, $\bL \sqsubseteq \csuf$ and 
     $\csuf$ satisfy Non-Triviality. 
     From Proposition~\ref{prop:csuf2},  $\ssuf \sqsubseteq \csuf$. Then, $\ssuf$ satisfies 
     Non-Triviality. Consequently, every $\bL \sqsubseteq \ssuf$ satisfies the axiom.

    \item \textit{Equivalence.} Consider Example~\ref{ex1} and let $\bQ_i = \langle\mbb{T}, \kappa, x_i\rangle$, with $i = 1, \ldots, 4$. Note that $\kappa(x_2) = \kappa(x_3)$ while 
    $\csuf(\bQ_2) = \{\{(f_1, 1)\}, \{(f_2, 0)\}\}$ while 
    $\csuf(\bQ_3) = \{\{(f_1, 0)\}, \{(f_2, 1)\}\}$. This shows that $\csuf$ violates Equivalence. 

    In Example~\ref{ex1}, $\kappa(x_1) = \kappa(x_4)$ but
    $\ssuf(\bQ_1) \neq \ssuf(\bQ_4)$.
    This shows that $\ssuf$ violates Equivalence. 
    
    Regarding $\bL \sqsubseteq z$, with $z \in \{\ssuf, \csuf\}$, 
    $\bL$ trivially satisfies Equivalence when $\forall \bQ, \bL(\bQ) = \emptyset$, but 
    $\bL$ violates Equivalence when $\bL = z$.

	\item \textit{Feasibility.} $\csuf$ violates the axiom since it satisfies Success and Weak-Validity, which are incompatible with Feasibility (see Theorem~\ref{incompatibility}). 

    $\ssuf$ violates the axiom since it satisfies Success and Strong Validity, which are incompatible with Feasibility (set Theorem~\ref{incompatibility}). 

    Regarding $\bL \sqsubseteq z$, with $z \in \{\ssuf, \csuf\}$, 
    $\bL$ is satisfied when $\forall \bQ, \bL(\bQ) = \emptyset$. 
    $\bL$ violates the axiom when $\bL = z$.

    \item \textit{Coreness.} $\csuf$ and $\ssuf$ violate Coreness since they both violate 
    Feasibility and from Proposition~\ref{links}, we have the implication: 
    Coreness $\Rightarrow$ Feasibility. 

    Regarding $\bL \sqsubseteq z$, with $z \in \{\ssuf, \csuf\}$, 
    $\bL$ is satisfied when $\forall \bQ, \bL(\bQ) = \emptyset$. 
    $\bL$ violates the axiom when $\bL = z$.

   \item \textit{Sceptical Validity.} Let $\bL \sqsubseteq z$, where $z \in \{\ssuf,\csuf\}$. 
   If $\forall \bQ$, $\bL(\bQ) = \emptyset$, then $\bL$ satisfies the axiom in a trivial way. 
   Assume now that $E \in \bL(\bQ)$. From Novelty, $E \cap x = \emptyset$. 
   Consequently, $x\circleddash E = \emptyset$, which means the axiom is satisfied in a trivial way by $\bL$. It follows that $\ssuf$ and $\csuf$ satisfy the axiom.

	\item \textit{Novelty.} From Theorems~\ref{cara:ssuf1} and \ref{coro:csuf1}.

	\item \textit{Strong Validity.} For $\bL \sqsubseteq \ssuf$ and $\ssuf$, the property follows from Theorem~\ref{cara:ssuf1}. 
	
	$\csuf$ violates the axiom as shown for the CSR  
    $\{(a,\text{skiing})\}$ of $\bQ_3$ in Example~\ref{ex1bis}. 
    The function $\bL \sqsubseteq \csuf$, such that  
    $\forall \bQ, \bL(\bQ) = \emptyset$, satisfies the axiom.

	\item \textit{Weak Validity.} For $\bL \sqsubseteq \ssuf$ and $\ssuf$, the property follows from satisfaction of Strong Validity and Proposition~\ref{links}. 
	For $\bL \sqsubseteq \csuf$ and $\csuf$, the property follows from Theorem~\ref{coro:csuf1}.

\end{itemize}
\end{proof}

\begin{theorem}\label{local:th:cara1}
An explainer $\bL$ satisfies Feasibility and Sceptical Validity \textbf{iff} $\bL \sqsubseteq \wL$.
\end{theorem}

\begin{proof}{}
Let $\bL$ be an explainer. 
\begin{itemize}
\item Assume $\bL$ satisfies Feasibility and Sceptical Validity. From Definition~5, 
it follows that $\bL \sqsubseteq \wL$. 
\item Assume that $\bL \sqsubseteq \wL$. 
From Definition~5, for all $E \in \bL(\bQ)$, we have:
(1) $E \subseteq x$, so $\bL$ satisfies Feasibility, and 
(2) $\forall y \in x\circleddash E$, $\kappa(y) \neq \kappa(x)$, so 
$\bL$ satisfies Sceptical Validity.  
\end{itemize}
\end{proof}

\begin{theorem}\label{th1-sufc}
An explainer $\bL$ satisfies Strong Validity \textbf{iff} $\bL \sqsubseteq \sufc$.	
\end{theorem}

\begin{proof}{}
Let $\bL$ be an explainer and $\bQ = \langle \mbb{T}, \kappa,x\rangle$ a query. 

\begin{itemize}
\item Assume that $\bL$ satisfies Strong Validity. 
      From the definition of Strong Validity, 
      for any $E \in \bL(\bQ)$, 
      $\forall y \in \mbb{F}(\mbb{T})$ such that 
      $E \subseteq y$, $\kappa(y) \neq \kappa(x)$. Hence, $E \in \sufc(\bQ)$. 
      So, $\bL \sqsubseteq \sufc$.	
\item Assume that  $\bL \sqsubseteq \sufc$.	
      Let $E \in \bL(\bQ)$.
      From Definition~6;
      $\forall y \in \mbb{F}(\mbb{T})$ such that $E \subseteq y$, $\kappa(y) \neq \kappa(x)$, so 
      $\bL$ satisfies Strong Validity. 
\end{itemize}
\end{proof}


\begin{theorem}\label{cara:ssuf1}
  An explainer $\bL$ satisfies Novelty and Strong Validity iff $\bL \sqsubseteq \ssuf$. 
\end{theorem}

\begin{proof}
Let $\bL$ be an explainer and $\bQ = \langle \mbb{T}, \kappa,x\rangle$ a query.    

\begin{itemize}
\item Assume that $\bL$ satisfies Novelty and Strong Validity. Let $E \in \bL(\bQ)$. 
From Novelty, $E \cap x = \emptyset$. 
From Strong Validity, $\forall y \in \mbb{F}(\mbb{T})$ such that $E \subseteq y$, $\kappa(y) \neq \kappa(x)$. 
Hence, $E \in \ssuf(\bQ)$ and so $\bL \sqsubseteq \ssuf$.   
\item Assume that  $\bL \sqsubseteq \ssuf$.	
      For any $E \in \bL(\bQ)$, 
      from Definition~7: (1) $E \cap x = \emptyset$ and 
      (2) $\forall y \in \mbb{F}(\mbb{T})$ such that $E \subseteq y$, $\kappa(y) \neq \kappa(x)$. So, $\bL$ satisfies the two axioms.   
\end{itemize}
\end{proof}

\begin{theorem}\label{coro:csuf1}
An explainer $\bL$ satisfies Novelty and Weak Validity iff $\bL \sqsubseteq \csuf$.  
\end{theorem}
\begin{proof}
Let  $\bL$ be an explainer and $\bQ = \langle \mbb{T}, \kappa,x\rangle$ a query. 

\begin{itemize}
\item 
Assume that $\bL$ satisfies Novelty and Weak Validity.  
Let $E \in \bL(\bQ)$. 
From Novelty, $E \cap x = \emptyset$. 
From Weak Validity, $\kappa(x_{\downarrow E}) \neq \kappa(x)$. 
From Property~\ref{propy:csuf}, $E \in \csuf(\bQ)$. 
So, $\bL \sqsubseteq \csuf$. 
\item Assume that $\bL \sqsubseteq \csuf$.     
Let $E \in \bL(\bQ)$. From Definition~8, 
$E = y\setminus x$ where $y \in \mbb{F}(\mbb{T})$. 
So, $E \in \mbb{E}(\mbb{T})$ and  $E \cap x = \emptyset$ which shows Novelty. 
Note that $y = E \cup (x \cap y)$. Hence, $y =  x_{\downarrow E}$. 
So, by Definition~8, $\kappa( x_{\downarrow E}) \neq \kappa(x)$, which ensures Weak Validity. 
\end{itemize}
\end{proof}


\begin{theorem}  
	An explainer $\bL$ satisfies Success, Novelty and Weak Validity iff 
	there exists a faithful ranking $\bS$ mapping every query $\bQ = \langle\mbb{T},\kappa,x\rangle$ 
 to a preorder 
	$\succeq_\bQ$ on $\mbb{E}(\mbb{T})$ s.t.   
	$$\bL(\bQ) = \max(\mbb{E}(\mbb{T}), \succeq_\bQ).$$
\end{theorem}

\begin{proof}
	Let $\bQ = \langle\mbb{T},\kappa,x\rangle$ be a query and $\bL$ an explainer.
	  
	Let $\mbb{E}(\mbb{T}) = \mbb{E}_1(\mbb{T}) \cup \mbb{E}_2(\mbb{T})$, where 
	\begin{itemize}
		\item $\mbb{E}_1(\mbb{T}) = \csuf(\bQ) = \{E \mid E \cap x = \emptyset \mbox{ and } \kappa(x_{\downarrow E}) \neq \kappa(x)\}$,
		\item $\mbb{E}_2(\mbb{T}) = \{E \mid E \cap x \neq \emptyset \mbox{ or } \kappa(x_{\downarrow E}) = \kappa(x)\}$. 
	\end{itemize}
	So, $\mbb{E}_1(\mbb{T}) \cap \mbb{E}_2(\mbb{T}) = \emptyset$.

	\begin{itemize}
		\item Assume that $\bL$ satisfies the three axioms and let $E \in \bL(\bQ)$ (which exists by Success). \\
		      By the definition of an explainer, $E \in \mbb{E}(\mbb{T})$ (1), \\
		      Novelty implies $E \cap x = \emptyset$ (2),  \\
		      Weak Validity implies $\kappa(x_{\downarrow E}) \neq \kappa(x)$ (3). \\
		      So, $\bL(\bQ) \subseteq \mbb{E}_1(\mbb{T})$.
		       
              Consider now a function $\bS$ that maps every query $\bQ$ to a \textbf{partial preorder} $\succeq_\bQ$ on $\mbb{E}(\mbb{T})$ that satisfies the following condition:  
              $$\forall E, E' \in \mbb{E}(\mbb{T}), \big(E \in \mbb{E}_1(\mbb{T})\big) \wedge \big(E' \in \mbb{E}_2(\mbb{T})\big) \Rightarrow E \succ_\bQ E'. $$ 
              and such that $\bL(\bQ) = \max(\mbb{E}(\mbb{T}), \succeq_\bQ)$.
              This is possible since
              $\bL(\bQ) \subseteq \mbb{E}_1(\mbb{T})$
              and $\bL(\bQ) \neq \emptyset$. 
			  Because of the the above property, the function $\bS$ is a faithful ranking. 
         
		\item Assume a faithful ranking $\bS$ that maps any query $\bQ$ to a preorder 
		$\succeq_\bQ$ on $\mbb{E}(\mbb{T})$ and such that  
		$\bL(\bQ) = \max(\mbb{E}(\mbb{T}),\succeq)$. Clearly,
  $\bL$ satisfies Success (since there is always at least one maximal element
  of a non-empty set).
  By definition of a faithful ranking, $\succeq_\bQ$ satisfies the condition below.  
		$$\forall E, E' \in \mbb{E}(\mbb{T}), \big(E \in \mbb{E}_1(\mbb{T})\big) \wedge \big(E' \in \mbb{E}_2(\mbb{T})\big) \Rightarrow E \succ_\bQ E'.$$ 
		
		Since $\mbb{E}_1(\mbb{T}) \neq \emptyset$ (by surjectivity of $\kappa$), it follows that $\bL(\bQ) \subseteq \mbb{E}_1(\mbb{T})$. 
		So, $\bL$ satisfies Novelty and Weak Validity.		 
	\end{itemize}
\end{proof}


\begin{theorem}\label{lem4}
The following properties hold:
\begin{itemize}
	\item The function $\bS$ mapping every query $\bQ$ to $\succeq^f_\bQ$ on $\mbb{E}(\mbb{T})$ is a faithful ranking. 
	\item For any query $\bQ$, $\bLwf(\bQ) = \max(\mbb{E}(\mbb{T}), \succeq^f_\bQ)$.
\end{itemize}	
\end{theorem}

\begin{proof}
Let $\bQ$ be a query. 
From Definition 11, $\succeq^f_{\bQ}$ is a partial preorder. 

\begin{itemize}
\item Let $E, E' \in \mbb{E}(\mbb{T})$ be such that: 
     \begin{itemize}
        \item $E \cap x = \emptyset$ and $\kappa(x_{\downarrow E}) \neq \kappa(x)$  and  
        \item $E' \cap x \neq \emptyset$  or $\kappa(x_{\downarrow E'}) = \kappa(x)$.   
    \end{itemize}
    Then, $\delta(E) = 1$ and $\delta(E') = +\infty$, 
    so $E \succ^f_{\bQ} E'$, which means that 
    $$\forall E \in \csuf(\bQ), \forall E' \notin \csuf(\bQ), E \succ^f E'. \qquad (1)$$
    This shows faithfulness of the function $\bS$ mapping every query $\bQ$ to a preorder $\succeq^f_{\bQ}$ on $\mbb{E}(\mbb{T})$.

	\item We show the equality $\bQ$, $\bLwf(\bQ) = \max(\mbb{E}(\mbb{T}), \succeq^f_{\bQ})$. 
   
    $\max(\mbb{E}(\mbb{T}), \succeq^f_{\bQ}) = \{E \in \mbb{E}(\mbb{T}) \mid \nexists E' \in \mbb{E}(\mbb{T}) \ s.t. \ E' \succ^f_{\bQ} E\}$. 
    From (1) it follows that $\max(\mbb{E}(\mbb{T}), \succeq^f_{\bQ}) \subseteq \csuf(\bQ)$. 
    So, $\max(\mbb{E}(\mbb{T}), \succeq^f_{\bQ}) = \{E \in \csuf(\bQ) \mid \nexists E' \in \csuf(\bQ) \ s.t. \ E' \succ^f_{\bQ} E\}$. 
    Since $\forall E, E' \in \csuf(\bQ)$, $\delta(E) = \delta(E')$, 
    $\max(\mbb{E}(\mbb{T}), \succeq^f_{\bQ}) = \{E \in \csuf(\bQ) \mid \nexists E' \in \csuf(\bQ) \ s.t. \ \feat(E') \subseteq \feat(E)\} = \bLwf(\bQ)$.
\end{itemize}	
\end{proof}

\begin{theorem} 
	The following properties hold:
	\begin{itemize}
		\item The function $\bS$ mapping every query $\bQ$ to $\succeq^c_\bQ$ on $\mbb{E}(\mbb{T})$ is a faithful ranking. 
		\item For any query $\bQ$, $\bL_c(\bQ) = \max(\mbb{E}(\mbb{T}), \succeq^c_\bQ)$.
	 \end{itemize}	
\end{theorem}

\begin{proof}
From Definition 13, $\succeq^c_\bQ$ is a total preorder. 
Let now $\bQ$ be a query. 

\begin{itemize}
\item Let $E, E' \in \mbb{E}(\mbb{T})$ be such that: 
     \begin{itemize}
        \item $E \cap x = \emptyset$ and $\kappa(x_{\downarrow E}) \neq \kappa(x)$  and  
        \item $E' \cap x \neq \emptyset$  or $\kappa(x_{\downarrow E'}) = \kappa(x)$.   
    \end{itemize}
    Then, $\sigma(E) = |E|$ and $\sigma(E') = +\infty$, so $E \succ^c E'$ (since the set of features is finite, $|E| < +\infty$), which shows faithfulness of the function $\bS$ mapping every query $\bQ$ to $\succeq^c_\bQ$ on $\mbb{E}(\mbb{T})$.

	\item We show the equality $\bQ$, $\bL_c(\bQ) = \max(\mbb{E}(\mbb{T}), \succeq^c_\bQ)$. 
   
    $\max(\mbb{E}(\mbb{T}), \succeq^c_\bQ) = \{E \in \mbb{E}(\mbb{T}) \mid \nexists E' \in \mbb{E}(\mbb{T}) \ s.t. \ E' \succ^c E\}$ = 
$\{E \in \mbb{E}(\mbb{T}) \mid \forall E' \in \mbb{E}(\mbb{T}), E \succeq^c E'\}$ (since $\succeq^c$ is a total preorder).  
So, $$\max(\mbb{E}(\mbb{T}), \succeq^c_\bQ) = 
\{E \in \mbb{E}(\mbb{T}) \mid \forall E', \sigma(E) \leq \sigma(E')\}.$$ 
Since $\forall E \in \csuf(\bQ)$, $\sigma(E) = |E|$ and 
$\forall E' \notin \csuf(\bQ)$, $\sigma(E') = +\infty$, 
we have $\forall E \in \csuf(\bQ)$, $\forall E' \notin \csuf(\bQ)$, $E \succ^c E'$. Hence, $\max(\mbb{E}(\mbb{T}), \succeq^c_\bQ) \subseteq \csuf(\bQ)$. So, 
$\max(\mbb{E}(\mbb{T}), \succeq^c_\bQ) = \{E \in \csuf(\bQ) \mid \nexists E' \in \csuf(\bQ) \ s.t. \ E' \succ^c E\} = \bL_c(\bQ)$.
\end{itemize} 
\end{proof}

\begin{theorem}  
Let $\dd$ be distance measure.
\begin{itemize}
	\item The function $\bS$ mapping every query $\bQ$ to $\succeq^d_{\bQ,\dd}$ on $\mbb{E}(\mbb{T})$ is a faithful ranking. 
	\item For any query $\bQ$, $\bL_\dd(\bQ) = \max(\mbb{E}(\mbb{T}), \succeq^d_{\bQ,\dd})$.
\end{itemize}	
\end{theorem}

\begin{proof}
Let $\bQ$ be a query and $\dd$ a distance measure on $\mbb{F}(\mbb{T})$. 
Hence, $$\dd: \mbb{F}(\mbb{T})^2 \rightarrow [0, k],$$ 
where $k < +\infty$. 
From Definition 15, $\succeq^d_{\bQ,\dd}$ is a total preorder.

\begin{itemize}
\item Let $E, E' \in \mbb{E}(\mbb{T})$ be such that: 
     \begin{itemize}
        \item $E \cap x = \emptyset$ and $\kappa(x_{\downarrow E}) \neq \kappa(x)$  and  
        \item $E' \cap x \neq \emptyset$  or $\kappa(x_{\downarrow E'}) = \kappa(x)$.   
    \end{itemize}
    Then, $\gamma(E) = \dd(x_{\downarrow E},x)$ and $\sigma(E') = +\infty$, 
    so $E \succ^d_{\bQ,\dd} E'$, which shows faithfulness of the function $\bS$ mapping every query $\bQ$ to $\succeq^d_{\bQ,\dd}$ on $\mbb{E}(\mbb{T})$.

	\item We show the equality $\bQ$, $\bL_c(\bQ) = \max(\mbb{E}(\mbb{T}), \succeq^d_{\bQ,\dd})$. 
   
    $\max(\mbb{E}(\mbb{T}), \succeq^d_{\bQ,\dd}) = \{E \in \mbb{E}(\mbb{T}) \mid \nexists E' \in \mbb{E}(\mbb{T}) \ s.t. \ E' \succ^d_{\bQ,\dd} E\}$ = 
$\{E \in \mbb{E}(\mbb{T}) \mid \forall E' \in \mbb{E}(\mbb{T}), E \succeq^d_{\bQ,\dd} E'\}$ (since $\succeq^d_{\bQ,\dd}$ is a total preorder).  
So, $\max(\mbb{E}(\mbb{T}), \succeq^d_{\bQ,\dd}) = 
\{E \in \mbb{E}(\mbb{T}) \mid \forall E', \sigma(E) \leq \sigma(E')\}.$ 
Since $\forall E \in \csuf(\bQ)$, $\gamma(E) = \dd(x_{\downarrow E},x)$ and 
$\forall E' \notin \csuf(\bQ)$, $\gamma(E') = +\infty$, we have
$\forall E \in \csuf(\bQ)$, $\forall E' \notin \csuf(\bQ)$, $E \succ^c E'$. Thus, it follows that $\max(\mbb{E}(\mbb{T}), \succeq^d_{\bQ,\dd}) \subseteq \csuf(\bQ)$. So, 
$\max(\mbb{E}(\mbb{T}), \succeq^d_{\bQ,\dd}) = \{E \in \csuf(\bQ) \mid \nexists E' \in \csuf(\bQ) \ s.t. \ E' \succ^d_{\bQ,\dd} E\} = \bL_d(\bQ)$.
\end{itemize}	
\end{proof}


\begin{theorem}  
Consider query $\bQ = \langle \mbb{T}, \kappa, x\rangle$ where
$\mbb{T}$ is composed of boolean features and $\kappa$ is a boolean formula.
\begin{enumerate}
\item The problem {\sc DecideExp}$(\bL)$ is $O(n)$ when $\bL$ is $\csuf$, $\wL$, or $\bLwf$.
\item The problem {\sc DecideExp}$(\bL)$ is co-NP-complete when $\bL$ is $\necp$, $\sufc$, 
$\ssuf$, $\bLa$, $\bL_d$.
\end{enumerate}
For domain sizes which are arbitrary, 
{\sc DecideExp}$(\bL)$ is $O(n)$ when $\bL$ is $\csuf$ 
or $\bLwf$.
For queries over non-boolean domains, {\sc DecideExp}$(\bL)$ is co-NP-complete
when $\bL$ is $\wL$. 
\end{theorem}

\begin{proof}
\begin{enumerate}
\item[]
\item  
$\csuf$ is the set of all $y\setminus x$ such that $\kappa(y) \neq \kappa(x)$.
Given $E$, it is polynomial-time to check whether there is an instance $y$
such that $E=y \setminus x$ and $\kappa(y)\neq \kappa(x)$. Indeed, such an instance $y$ can only exist
if $x \cap E = \emptyset$, in which case it is necessarily $x_{\downarrow E}$.
Thus {\sc DecideExp}($\csuf$) is $O(n)$. 

We now consider $\wL$. 
For $E \subseteq x$,
$E \in \wL(\bQ)$ if for all $y$ such that $x \setminus y = E$, $\kappa(y)\neq \kappa(x)$.
Let $\overline{E}$ denote $\{ (f,\overline{v}) \mid (f,v) \in E \}$. When all feature domains
are boolean, $x \setminus y = E$ only for $y = x_{\downarrow \overline{E}}$.
Thus {\sc DecideExp}($\bL$) is $O(n)$ for $\bL = \wL$. 
since, given $E$, it is $O(n)$ to check whether  $E \subseteq x$ and 
$\kappa( x_{\downarrow \overline{E}}) \neq \kappa(x)$.

We now consider $\bLwf$.  $E \in \bLwf(Q)$ iff $\kappa(x_{\downarrow E}) \neq \kappa(x)$
{\sc DecideExp}($\bLwf$) is $O(n)$, since given $E$ it suffices to test $\kappa(x_{\downarrow E}) \neq \kappa(x)$
which is $O(n)$.

\item For the cases $\necp$, 
$\sufc$ and $\ssuf$,
to show that the problem belongs to co-NP, it suffices to observe that a counterexample,
 i.e. an instance $y$ such that $\kappa(y)=\kappa(x)$, can be verified in polynomial time. 
To show co-NP-completeness, we consider each of $\necp$, 
$\sufc$ and $\ssuf$ in turn.
It suffices, in each case, to exhibit a polynomial reduction from {\sc Tautology}, the problem
of determining whether a DNF $\Psi(f_1,\ldots,f_n)$ is a tautology. 
We first consider $\necp$. 
Define $\kappa(f_1,\ldots f_{n+1}) =\Psi(f_1,\ldots,f_n) \land f_{n+1}$. Let $x=(0,\ldots,0)$ 
(so that $\kappa(x)=0)$ and let $E=\{(f_{n+1},0)\}$. 
$E \in \necp(\bQ)$ iff for all $y$ which doesn't contain $E$, $\kappa(y)\neq \kappa(x)$,
i.e. iff $\Psi(f_1,\ldots,f_n) \equiv 1$. This is a polynomial reduction from {\sc Tautology}
to {\sc DecideExp}($\necp$). 
We now consider  $\sufc$ and $\ssuf$.
We use the same $\kappa$ and $x$ but $E=\{(f_{n+1},1)\}$.
$E \in \ssuf(\bQ)$ iff $E \cap x = \emptyset$ (which is the case)
and for all $y$ containing $E$, $\kappa(y)\neq \kappa(x)$.
So $E \in \ssuf(\bQ)$ iff $\Psi(f_1,\ldots,f_n) \equiv 1$. This provides a 
polynomial reduction from {\sc Tautology} to {\sc DecideExp}($\ssuf$).  The proof is identical for $\sufc$.

We now consider $\bLa$ and $\bL_d$. 
$\bLa(\bQ)$ and $\bL_d(\bQ)$ contain
the explanations in $\csuf(\bQ)$ which are minimal according to cardinality or the distance $d$,
respectively. 
Since membership of $\csuf(\bQ)$ 
can be tested in polynomial time, a counter-example certificate 
(which can either be $E$ itself, if $E \notin \bL$, or another $E' \in \bL$ which is preferred to $E$) 
can be verified in polynomial time. This shows membership of co-NP for {\sc DecideExp}($\bL$) when
$\bL$ is $\bLa$ or $\bL_d$. 
To demonstrate co-NP-hardness, consider 
$\kappa(f_1,\ldots,f_{n+1}) = (\Phi(f_1,\ldots,f_n) \land f_{n+1}) \lor (f_1 \land \ldots \land f_{n+1})$,
where $\Phi(f_1,\ldots,f_n)$ is a CNF. Let $x=(0,\ldots,0)$, so that $\kappa(x)=0$, 
and let $E=\{(f_1,1),\ldots,(f_{n+1},1)\}$. $E \in \csuf(\bQ)$ and $E \in \bLwf(\bQ)$.
There exists another explanation in $\csuf(\bQ)$ 
(which if it exists will necessarily be preferred according to 
cardinality or distance $d$)
iff $\Phi$ is satisfiable.
We therefore have a polynomial reduction from {\sc UNSAT} to {\sc DecideExp}($\bL$)
when $\bL$ is  $\bLa$ and $\bL_d$, 
demonstrating co-NP-hardness of these 
two problems.
\end{enumerate}
The proofs that {\sc DecideExp}($\bQ$) is $O(n)$ for arbitrary domain sizes 
when $\bL$ is $\csuf$
or $\bLwf$ are identical to the proofs above in the boolean case.

We now consider {\sc DecideExp}($\wL$) for non-boolean domains.
The problem belongs to co-NP since a counter-example is an instance
$y$ differing from $x$ exactly on $E$ such that $\kappa(y) = \kappa(x)$.
To complete the proof of
co-NP-completeness, it suffices to give a polynomial reduction from {\sc Tautology}.
Let $\Psi$ be a DNF on $n$ variables  $f_1,\ldots,f_n$. Consider a theory consisting of 
$n$ features with domains $\{0,1,2\}$. Let $x$ be the instance $(2,\ldots,2)$ and let $E=x$.
Define $\kappa$ so that $\kappa(2,\ldots,2)=0$, $\kappa(z)=\phi(z)$ for $z \in \{0,1\}^n$
(with $\kappa$ taking arbitrary values on other instances). Then $E \in \wL(\bQ)$
(where $\bQ$ is the query corresponding to $\kappa$ and $x$) 
iff $\Psi \equiv 1$, which completes the polynomial reduction from {\sc Tautology}.
\end{proof}

\begin{theorem}  
Consider query $\bQ = \langle \mbb{T}, \kappa, x\rangle$ where
$\mbb{T}$ is composed of $n$ boolean features and $\kappa$ is a boolean formula.
\begin{enumerate}
\item The problem {\sc FindExp}$(\bL)$ is $O(n)$ when $\bL$ is $\ssuf$. 
\item The problem {\sc FindExp}$(\bL)$ is NP-hard when $\bL$ is 
$\necp$, $\csuf$, $\wL$, $\sufc$, 
$\bLa$, $\bL_d$, $\bLwf$.
In each of these cases, excluding $\bLa$ and $\bL_d$,
it can be solved by at most $n$ calls to a SAT oracle. 
\end{enumerate}
For queries over non-boolean domains {\sc FindExp}$(\bL)$ is NP-hard for $\ssuf$.
\end{theorem}

\begin{proof}
\begin{enumerate}
\item 
We consider the case of $\ssuf$.
Let $\overline{x}$ denote $\{ (f,\overline{v}) \mid (f,v) \in x \}$.
If $E \in \ssuf(\bQ)$, then $E \cap x = \emptyset$ and for all $y \supseteq E$, $\kappa(y)\neq \kappa(x)$
and hence $\kappa(\overline{x})\neq \kappa(x)$, so $\overline{x} \in \ssuf(\bQ)$.
Thus to find an explanation in $\ssuf$ (or show that none exists) it suffices to test $\kappa(\overline{x})\neq \kappa(x)$
which is $O(n)$.

\item 
We first consider $\necp$. Let $\Phi(f_1,\ldots,f_n)$ be a CNF. Define $\kappa(f_1,\ldots,f_{n+1})=
\Phi(f_1,\ldots,f_n) \lor f_{n+1}$, and let $x=(1,\ldots,1)$ and $E=\{(f_{n+1},1)\}$.
$E \in \necp(\bQ)$ iff for all $y$, $\kappa(y)=\kappa(x)$ only if $E \subseteq y$.
In other words, $E \in \necp(\bQ)$ iff $\Phi$ is unsatisfiable. This polynomial reduction from
{\sc SAT} demonstrates the NP-hardness of {\sc FindExp}($\necp$).
In general, finding some $E \in \necp(\bQ)$ (or determining that none exists) can be achieved by looking for a
singleton $E=\{\ell\}$ where $\ell$ is a core literal. Since $\kappa$ is a boolean formula,
this can be achieved by $n$ calls to a SAT oracle.

We now consider $\csuf$. $E \in \csuf(\bQ)$ iff $E=y \setminus x$ for some instance $y$ such that 
$\kappa(y)\neq \kappa(x)$. Let $\Phi(f_1,\ldots,f_n)$ be a CNF.
Define $\kappa(f_1,\ldots,f_{n+1})= \Phi(f_1,\ldots,f_n) \land f_{n+1}$ and let $x=(0,\ldots,0)$, so that 
$\kappa(x)=0$. Finding $E \in \csuf(\bQ)$ (where $\bQ$ is the corresponding query)
amounts to finding a satisfying assignment for $\Phi$. This polynomial reduction from
{\sc SAT} demonstrates the NP-hardness of {\sc FindExp}($\csuf$).
In general, finding an explanation in $\csuf(\bQ)$ amounts to finding an instance $y$
such that $\kappa(y)\neq \kappa(x)$ and hence can be achieved by a single call to a SAT oracle.

We now consider 
$\wL$. Over boolean domains, $E \in \wL(\bQ)$ iff $E \subseteq x$ and there is some instance $y$ such that
$x \setminus y = E$ and $\kappa(y)\neq \kappa(x)$. Let $\Phi(f_1,\ldots,f_n)$ be a CNF. 
Define $\kappa(f_1,\ldots,f_{n+1})= \Phi(f_1,\ldots,f_n) \land f_{n+1}$ and let $x=(0,\ldots,0)$
so that $\kappa(x)=0$. Finding $E \in \wL(\bQ)$, where $\bQ$ is the corresponding query,
amounts to finding a satisfying assignment for $\Phi$. This polynomial reduction from
{\sc SAT} demonstrates the NP-hardness of {\sc FindExp}($\wL$).

We now consider $\sufc$. 
$E \in \sufc(\bQ)$ if for all $y$ containing $E$, $\kappa(y)\neq \kappa(x)$.
It is easy to see that to find such an $E$, it suffices to look for complete instances $E=y$
such that $\kappa(y)\neq \kappa(x)$.
Let $\Phi(f_1,\ldots,f_n)$ be a CNF and define $\kappa(f_1,\ldots,f_{n+1})= \Phi(f_1,\ldots,f_n) \land f_{n+1}$.
Let $x=(0,\ldots,0)$ (so that $\kappa(x)=0$). Any $y$ such that $\kappa(y)\neq  \kappa(x)$
is a solution to the SAT instance given by $\Phi$, which demonstrates the NP-hardness of {\sc FindExp}($\sufc$).

We now show that 
a single call
to a SAT oracle is 
sufficient to solve {\sc FindExp}($\bL$)
when $\bL$ is $\sufc$. 
$E \in \sufc(\bQ)$ if for all $y$ containing $E$, $\kappa(y)\neq \kappa(x)$.
As we have just observed, it suffices to look for an instance $y$ such that $\kappa(y) \neq \kappa(x)$,
which can be achieved by a single call to a SAT oracle.

We now consider $\bLa$ and $\bL_d$.
NP-hardness follows
directly from the proof that {\sc DecideExp}($\bL$) is co-NP-complete when 
$\bL$ is $\bLa$ or $\bL_d$ 
(Theorem~\ref{thm:dec}).
In the case of $\bLwf$, it suffices to consider $\kappa(f_1,\ldots,f_{n+1}) =
\Phi(f_1,\ldots,f_n) \land f_{n+1}$, $x=(0,\ldots,0)$ (so that $\kappa(x)=0$).
Finding an explanation $E \in \bLwf$ amounts to finding a satisfying assignment
to the CNF $\Phi$. Hence, {\sc DecideExp}($\bLwf$) is also NP-hard.

We now show that {\sc FindExp}($\bL$) can be solved using SAT
when $\bL$ is $\bLwf$. 
To find $E \in \bLwf$ it suffices to find
an instance $y$ such that $\kappa(y) \neq \kappa(x)$ and set $E=y \setminus x$.
This can be achieved with a single call to a SAT oracle. 
\end{enumerate}

For queries over non-boolean domains, we consider $\ssuf$. 
$E \in \ssuf(\bQ)$ iff $E \cap x = \emptyset$ and for all instances $y$ containing $E$, $\kappa(y)\neq \kappa(x)$.
Thus, if there is some $E \in \ssuf(\bQ)$ then there is also complete instance $y \in \ssuf(\bQ)$.
Consider a theory consisting of $n$ features with domains $\{0,1,2\}$.
Let $\Phi(f_1,\ldots,f_n)$ be a CNF, let $x=(2,\ldots,2)$ and define $\kappa$ so that
$\kappa(2,\ldots,2)=0$ and $\kappa(z)=\Phi(z)$ for $z \in \{0,1\}^n$.
Then finding an instance $y$ such that $y \cap x = \emptyset$ and $\kappa(y)\neq \kappa(x)$
amounts to finding a satisfying instance of $\Phi$. It follows that {\sc FindExp}($\ssuf$) is NP-hard.
\end{proof}



\end{document}